\documentclass[letterpaper]{article} 
\usepackage[preprint]{aaai2027}
\usepackage[hyphens]{url}  
\usepackage{graphicx} 
\usepackage{natbib}  
\usepackage{caption} 
\usepackage{amsmath}  
\usepackage{amssymb}  
\usepackage{amsthm}   
\newtheorem{theorem}{Theorem}
\newtheorem{corollary}{Corollary}
\newtheorem{proposition}{Proposition}

\newtheorem{lemma}{Lemma}
\usepackage{tikz}     
\usetikzlibrary{arrows.meta}
\definecolor{ringred}{HTML}{B33A3A}
\usepackage{algorithm}
\usepackage{algorithmic}
\usepackage{newfloat}
\usepackage{listings}
\DeclareCaptionStyle{ruled}{labelfont=normalfont,labelsep=colon,strut=off} 
\floatstyle{ruled}
\newfloat{listing}{tb}{lst}{}
\floatname{listing}{Listing}
\graphicspath{{figures/}}

\title{Mapping and Measuring the Behavioral Evolution of Large Language Models}
\author{
Dong Qiao,
Chris Ding,
Jicong Fan\corresponding
}
\affiliations{
School of Data Science, The Chinese University of Hong Kong, Shenzhen, China\\
dongqiao@link.cuhk.edu.cn, chrisding@cuhk.edu.cn, fanjicong@cuhk.edu.cn
}

\begin{document}

\maketitle

\begin{abstract}
Benchmark leaderboards summarize how well a language model performs, but not how
its behavior relates to that of other models or changes across generations. We
characterize the output behavior of 32 models from six families using their
responses to a shared bank of 10{,}000 prompts. After embedding each response,
we construct three complementary sentence-level dissimilarities: an aligned
mean per-prompt distance, which is a pseudometric on observed model responses; a
PCA-compressed summary of prompt-wise disagreement; and an alignment-free
Gromov--Wasserstein discrepancy between models' internal response geometries.
We use these constructions to study static organization and temporal change on
a release-date axis through behavioral maps, family-wise drift, hierarchical
clustering, cross-family convergence, and response-cloud dispersion. Across the
three constructions, model families form coherent clusters, with
\texttt{gpt-2} as a global outlier; cross-family distances decrease over time;
and several recent reasoning-oriented models have comparatively compact
response clouds. A token-level cross-check based on per-prompt Maximum Mean
Discrepancy closely agrees with the sentence-level mean distance (Spearman
$\rho=0.98$) and recovers the same qualitative findings. We organize these
comparisons through a measure-theoretic lens making their alignment and
invariance assumptions explicit. We also establish an architecture-agnostic sufficient condition linking behavioral similarity to inference-prompt coverage, small excess population log-loss, and similar effective target distributions---a possible training-side account rather than an empirical explanation of the observed trends. Our pipeline is
label-free, and re-encoding every response with three further encoders---down
to one $73\times$ smaller---preserves the rank geometry, the outliers, and the
sign of the time trend.
\end{abstract}

\section{Introduction}
The pace of large language model (LLM) releases has outstripped our tools for
understanding how models relate to one another. Standard evaluation compresses
a model's performance on each benchmark into a scalar score
\cite{liang2022holistic,srivastava2023beyond}. Such scores are essential for
measuring capability, but they do not reveal behavioral relationships
\cite{yax2025phylolm,goel2025great,zhang2026spectral}: whether two models answer in similar ways,
whether a new release departs from its predecessors, or whether models from
different vendors are becoming more alike. We therefore ask: \emph{Can the
behavioral evolution of LLMs be mapped and measured directly from their
outputs, without using ground-truth labels?}

We answer this question by applying the same 10{,}000-prompt bank to 32 models,
embedding their responses, and comparing the resulting response collections.
Our contributions are fivefold. First, we introduce three complementary
sentence-level constructions: an aligned mean per-prompt pseudometric, a
PCA-compressed summary of prompt-wise disagreement, and an alignment-free
Gromov--Wasserstein (GW) discrepancy over internal response geometry
\citep{memoli2011gw,peyre2019computational}. Second, we place the models on a
release-date axis and quantify both static structure and temporal change using
behavioral maps, drift curves, ordered heatmaps, hierarchical clustering,
cross-family distance, and response-cloud dispersion. Third, we provide an
empirical characterization of 32 models from six families and identify three
patterns that are stable across constructions: family coherence, decreasing
cross-family distance, and relatively compact response clouds for several
recent reasoning-oriented models. Fourth, we validate the sentence-level
picture with an independent token-level Maximum Mean Discrepancy (MMD) analysis
\citep{gretton2012mmd}. Fifth, we clarify the mathematical relationships among
the constructions and prove a training-side sufficient condition that connects
behavioral similarity to effective target-distribution similarity, excess
population log-loss, and inference-prompt coverage. We additionally quantify
tree-likeness and the low-dimensional structure of prompt-wise disagreement.

\section{Related Work}

\paragraph{Behavioral evaluation and black-box fingerprinting.}
Behavioral evaluation studies models through their observable responses rather
than only through aggregate benchmark scores. CheckList
\citep{ribeiro2020beyond} introduced systematic behavioral testing for NLP
models, while broader surveys organize model behavior across linguistic,
reasoning, knowledge, and social dimensions
\citep{chang2023language}. More recent work uses output behavior to identify or characterize LLMs. LLMmap \citep{pasquini2025llmmap} actively selects prompts
whose responses distinguish deployed model versions, whereas behavioral
fingerprinting \citep{pei2025behavioral} constructs multidimensional profiles
of cognitive and alignment-related tendencies. These methods primarily target
testing, identification, or trait profiling. Our objective is instead to
construct a continuous geometry over models from a shared, broad prompt bank
and use that geometry to measure lineage structure, temporal drift, and
cross-family convergence.

\paragraph{Are models converging?}
The Platonic Representation Hypothesis \citep{huh2024platonic} holds that networks trained on
different data and objectives are converging toward a shared statistical model of reality. The
claim is now contested from two directions. \citet{groger2026aristotelian} advance an Aristotelian
reading in which representations remain particularistic and task-dependent, reporting---via
permutation tests and canonical correlation analysis---that networks diverge more than the
universalist account predicts. \citet{koepke2026platoscave} re-examine cross-modal alignment at
scale and find that the mutual-nearest-neighbor agreement supporting convergence holds on small
samples but degrades substantially as the data grow to millions, leaving the evidence for a common
representation weaker than assumed.

Our analysis is complementary to this debate. We measure convergence in what models
\emph{emit}, rather than in how they \emph{encode}, and study change along a
release-date axis rather than across modalities or objectives. Behavioral
convergence is compatible with either conclusion about internal
representations: output distributions may become more similar even when hidden
representations remain distinct. Our GW analysis probes this distinction by
comparing response-space geometry rather than absolute embedding coordinates.

\paragraph{Comparing models by their internals.}
A parallel line compares models mechanistically. \citet{wang2025universality} use sparse
autoencoders to isolate interpretable features across Transformer \cite{vaswani2017attention} and Mamba \cite{gu2023mamba} architectures, find
most features shared between the two, and show that Mamba's induction circuits are structurally
analogous to their Transformer counterparts---evidence for universality at the level of learned
mechanism. \citet{lan2024universality} apply representational-similarity metrics to
sparse-autoencoder feature spaces across different LLMs and report substantial agreement, with
subspaces for particular semantic concepts aligning especially closely. Representation-similarity
analysis \citep{kornblith2019cka} compares internal activations in the same spirit. All of these
require white-box access to weights or activations. Our distances need only generated text, which
is what makes the closed frontier models in this study measurable at all.

\paragraph{Tracing model evolution from behavior.}
Closest to our setting is LLM DNA \citep{wu2026llmdna}, which derives a
low-dimensional, bi-Lipschitz encoding of functional behavior from model
responses and uses the induced distances to construct a phylogeny over many
models. Both studies compare models through outputs rather than weights, but
their goals differ. LLM DNA compresses behavior into a genotype-like
representation to recover ancestry at scale, primarily among open-weight
models. We retain pairwise behavioral geometry, compare constructions with
different alignment and aggregation assumptions, and focus on temporal change
among predominantly closed models: whether cross-family distances decrease
with release date and how closely the resulting geometry resembles a tree.

\section{Data and Methods}

\paragraph{Models and prompts.}
Let $M=32$ denote the number of evaluated models. They span six families: GPT
\citep{radford2019gpt2}, Claude \citep{anthropic2024claude3}, Gemini
\citep{geminiteam2023gemini}, Qwen \citep{qwen2024qwen25}, Llama
\citep{grattafiori2024llama3}, and Mistral \citep{jiang2023mistral}. Every model
answers the same canonical bank of $N{=}10{,}000$ prompts drawn from 13 public
benchmarks, including MMLU \citep{hendrycks2021mmlu}, ARC
\citep{clark2018arc}, HellaSwag \citep{zellers2019hellaswag}, WinoGrande
\citep{sakaguchi2020winogrande}, GSM8K \citep{cobbe2021gsm8k}, MATH
\citep{hendrycks2021math}, and TruthfulQA \citep{lin2022truthfulqa}. Thus,
response index $i$ always refers to the same prompt across models. The complete
benchmark composition and model roster appear in the appendix
(Tables~\ref{tab:benchmarks} and~\ref{tab:roster}); Table~\ref{tab:models}
summarizes the six families and their release spans.

\begin{table}[t]
\centering
\small
\begin{tabular}{lccl}
\hline
Family & \# & Release span & Latest (this study)\\
\hline
GPT     & 5 & 2019--2025 & \texttt{gpt-5.2}\\
Claude  & 8 & 2025--2026 & \texttt{claude-opus-4-7-think}\\
Gemini  & 4 & 2024--2026 & \texttt{gemini-3.1-pro-preview}\\
Qwen    & 5 & 2024--2026 & \texttt{qwen3.7-max}\\
Llama   & 5 & 2024--2025 & \texttt{llama-4-maverick}\\
Mistral & 5 & 2024--2025 & \texttt{mistral-large-2512}\\
\hline
\end{tabular}
\caption{The 32 evaluated models grouped by family. Dates marked as approximate
in Table~\ref{tab:roster} are used only for ordering and horizontal placement
on the time axis.}
\label{tab:models}
\end{table}

\paragraph{Response embeddings.}
We encode each complete response with Qwen3-Embedding-8B
\citep{zhang2025qwen3embedding} and mean-pool its token embeddings, following
the sentence-embedding paradigm \citep{reimers2019sbert}. The encoder has a
maximum input length of $1024$ tokens and produces a $4096$-dimensional vector;
we do not apply $L_2$ normalization. For model $m$, the resulting matrix
$E_m\in\mathbb{R}^{N\times4096}$ contains one response embedding per prompt and
is row-aligned with every other model's matrix.

\begin{figure*}[t]
\centering
\tikzset{
  lbl/.style={font=\scriptsize},
  dm/.style={circle, fill=blue!55, inner sep=1.5pt},
  dn/.style={circle, fill=orange!85, inner sep=1.5pt},
}
\begin{minipage}[t]{0.32\textwidth}\centering
\textbf{\small (a) Mean per-prompt $D^{\mathrm{mean}}$}\\[3pt]
\begin{tikzpicture}[baseline]
  \foreach \i/\y in {1/2.4, 2/1.8, 3/1.2, 4/0.6}{
    \node[dm] (m\i) at (0,\y){};
    \node[dn] (n\i) at (2.2,\y){};
    \draw[dashed, gray!55] (m\i) -- (n\i);
  }
  \draw[-{Latex}, thick] (m2) -- (n2);
  \node[lbl, above=1pt] at (1.1,1.8){$\lVert E_m[i]-E_{m'}[i]\rVert$};
  \node[lbl] at (0,2.85){$E_m$};
  \node[lbl] at (2.2,2.85){$E_{m'}$};
\end{tikzpicture}\\[3pt]
\end{minipage}
\hfill
\begin{minipage}[t]{0.32\textwidth}\centering
\textbf{\small (b) PCA-compressed $D^{\mathrm{pca}}$}\\[3pt]
\begin{tikzpicture}[baseline]
  \foreach \dx/\dy in {0.34/0.34, 0.17/0.17}{
    \draw[fill=violet!12] (\dx,\dy) rectangle ++(1.1,1.1);
  }
  \draw[fill=violet!22] (0,0) rectangle ++(1.1,1.1);
  \foreach \g in {0.275,0.55,0.825}{
    \draw[violet!45, very thin] (\g,0) -- (\g,1.1);
    \draw[violet!45, very thin] (0,\g) -- (1.1,\g);
  }
  \node[lbl] at (0.55,-0.3){per-prompt $32\times32$};
  \node[lbl] at (0.7,1.75){stacked over $N$};
  \draw[-{Latex}, thick] (1.75,0.7) -- (2.55,0.7);
  \node[lbl, above=1pt] at (2.15,0.75){PCA};
  \draw[fill=violet!35] (2.7,0.17) rectangle ++(1.1,1.1);
  \node[lbl] at (3.25,-0.05){$D$: PC$_1$ scores};
\end{tikzpicture}\\[3pt]
{\scriptsize variance-weighted compression over prompts}
\end{minipage}
\hfill
\begin{minipage}[t]{0.32\textwidth}\centering
\textbf{\small (c) Gromov--Wasserstein $D^{\mathrm{gw}}$}\\[3pt]
\begin{tikzpicture}[baseline,
  im/.style={circle, fill=teal!70, inner sep=1.3pt},
  inn/.style={circle, fill=red!65, inner sep=1.3pt}]
  \node[im] (a1) at (0,2.3){};
  \node[im] (a2) at (0.85,1.9){};
  \node[im] (a3) at (0.25,1.35){};
  \draw[teal!55] (a1)--(a2)--(a3)--(a1);
  \node[lbl] at (0.1,2.75){$C_m$};
  \node[inn] (b1) at (3.15,2.4){};
  \node[inn] (b2) at (3.85,1.85){};
  \node[inn] (b3) at (3.15,1.35){};
  \draw[red!50] (b1)--(b2)--(b3)--(b1);
  \node[lbl] at (3.95,2.8){$C_{m'}$};
  \draw[dashed, gray!70] (a1)--(b2);
  \draw[dashed, gray!70] (a2)--(b3);
  \draw[dashed, gray!70] (a3)--(b1);
  \node[lbl, gray!70] at (1.95,0.75){optimal coupling};
\end{tikzpicture}\\[-17pt]
{\scriptsize matches internal geometries; alignment-invariant}
\end{minipage}
\vspace{-8pt}
\caption{Three sentence-level comparison constructions. \textbf{(a)} The mean per-prompt distance compares embeddings of the
same prompt and averages the Euclidean distances; it is an aligned
pseudometric. \textbf{(b)} The PCA construction compresses each pair's vector
of prompt-wise distances onto the leading disagreement direction; the resulting
score need not be a metric. \textbf{(c)} GW compares internal response geometry
under an optimized coupling. With cosine costs, it is invariant to orthogonal
transformations and positive rescaling of individual embedding vectors, and it does not require prompt indices to
be matched.}
\label{fig:distances}
\end{figure*}
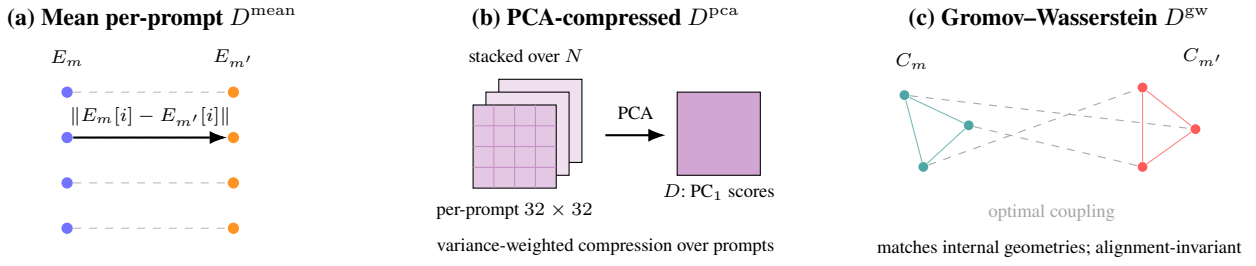

\paragraph{Cross-model dissimilarities.}
For models $m$ and $m'$, we construct an $M\times M$ pairwise matrix in three
ways (Fig.~\ref{fig:distances}).
\begin{itemize}
\item \emph{Mean per-prompt distance.} We preserve the known prompt alignment
and average Euclidean distances between paired response embeddings:
\begin{equation}
D^{\mathrm{mean}}_{m,m'}=\frac{1}{N}\sum_{i=1}^{N}
\lVert E_m[i]-E_{m'}[i]\rVert_2.
\label{eq:dmean}
\end{equation}
It is a pseudometric on models represented by their observed
embedding matrices; it becomes a metric after quotienting out models with
identical encoded responses on every prompt (Proposition~\ref{prop:pseudometric} in Appendix).

\item \emph{PCA-compressed disagreement.} For prompt $i$, let
$M^{(i)}\in\mathbb{R}^{M\times M}$ contain the pairwise distances
$M^{(i)}_{m,m'}=\lVert E_m[i]-E_{m'}[i]\rVert_2$. For each model pair, collect
its prompt-wise disagreement profile
$x_{m,m'}=(M^{(1)}_{m,m'},\ldots,M^{(N)}_{m,m'})^\top\in\mathbb{R}^{N}$.
Applying rank-one PCA to these pair profiles gives
\begin{equation}
s_{m,m'}=\bigl\langle x_{m,m'}-\bar{x},u_1\bigr\rangle,\
D^{\mathrm{pca}}_{m,m'}=\mathcal{S}(s_{m,m'}),
\label{eq:dpca}
\end{equation}
where $u_1$ is the leading principal direction and $\bar{x}$ is the mean across model pairs,
and $\mathcal{S}$ denotes the sign orientation, symmetrization, and rescaling
used to form the displayed matrix. It is a one-dimensional summary of how
pairs disagree across prompts, not necessarily a metric.

\item \emph{Gromov--Wasserstein discrepancy.} For model $m$, let
$C_m\in\mathbb{R}^{N\times N}$ be the matrix of cosine dissimilarities among
its own response embeddings, and let $\mathbf{u}_m$ be the uniform measure on
those responses. We compute the squared GW discrepancy
\begin{equation}
D^{\mathrm{gw}}_{m,m'}=
\min_{\pi\in\Pi(\mathbf{u}_m,\mathbf{u}_{m'})}
\sum_{a,b,c,d}
\bigl(C_m[a,b]-C_{m'}[c,d]\bigr)^2\pi_{ac}\pi_{bd},
\label{eq:dgw}
\end{equation}
where
$\Pi(\mathbf{u}_m,\mathbf{u}_{m'})=
\{\pi\geq0:\pi\mathbf{1}=\mathbf{u}_m,
\pi^\top\mathbf{1}=\mathbf{u}_{m'}\}$.
GW compares internal response geometry under an optimized coupling and therefore
does not require prompt indices to be matched. Because $C_m$ uses cosine
dissimilarity, the construction is unchanged by model-specific orthogonal
transformations and positive rescaling of the embedding vectors, but not by
translations. We solve Eq.~\eqref{eq:dgw} with POT's conditional-gradient
solver \citep{flamary2021pot}. Full $N\times N$ costs are computationally
prohibitive, so all models use the same seeded subsample of $256$ prompts.
\end{itemize}

\paragraph{A measure-theoretic comparison lens.}
Let $q_m(\cdot\mid x)$ denote model $m$'s response distribution for prompt $x$,
and let $\nu_m^x=\phi_{\#}q_m(\cdot\mid x)$ be its pushforward under the
response encoder $\phi$, e.g., Qwen3-Embedding-8B. The aligned constructions condition on the same prompt
$x$ for both models. At the sentence level, the population quantity naturally
associated with one independently generated response from each model is
\begin{equation}
\mathbb{E}_{x}\mathbb{E}_{(Z,Z')\sim\nu_m^x\otimes\nu_{m'}^x}
\|Z-Z'\|_2,
\label{eq:population-paired-distance}
\end{equation}
When the prompt bank is viewed as a sample from the inference distribution and responses are sampled independently, Eq.~\eqref{eq:dmean} is the corresponding Monte Carlo estimate. This
quantity differs from the distance between conditional means,
$\|\mathbb{E}Z-\mathbb{E}Z'\|_2$, unless decoding is deterministic or
within-prompt variability is negligible. The PCA construction does not define a
new per-prompt divergence; it compresses the vector of aligned per-prompt
dissimilarities into a single score. The token-level MMD analysis, to be introduced in Section \ref{sec_mmd}, remains
prompt-aligned but replaces each pooled response vector by the empirical measure
of its token embeddings. In contrast, GW discards the fixed prompt alignment
and compares only the internal geometry of each model's response cloud under an
optimized coupling.

Thus, the constructions form a hierarchy of comparison assumptions rather than
four instances of a single metric: they differ in representation (pooled
response versus token measure), aggregation (mean versus PCA compression), and
alignment (fixed prompts versus optimized coupling). Their high empirical rank
correlations therefore provide a robustness check: the principal organization
of the models is stable despite these changes in representation and alignment.
For token MMD, taking the square root recovers a metric on empirical
measures when the kernel is characteristic. We report the conventional squared
GW objective on cosine-dissimilarity networks; taking its square root leaves
all rank-based analyses unchanged, although a formal metric interpretation of
GW requires metric ground costs.

\paragraph{Visualization and temporal summaries.}
We visualize each pairwise matrix with metric multidimensional scaling (MDS)
\citep{kruskal1964mds}, t-SNE \citep{vandermaaten2008tsne}, and UMAP
\citep{mcinnes2018umap}; per-family panels reuse the corresponding global
coordinates. To study evolution, we associate each model with its release date
and report (i) cumulative drift from the first model in its family, (ii) step
drift between consecutive family releases, (iii) an ordered distance heatmap
and average-linkage dendrogram, (iv) mean distance to models from other
families, and (v) response-cloud dispersion, defined as the mean off-diagonal
entry of $C_m$. A few release dates are approximate and affect only the
horizontal placement of points, not any pairwise dissimilarity. The analysis is
related in spirit to representation-similarity methods
\citep{kornblith2019cka}, but uses model outputs rather than hidden activations.

\section{Results}

\begin{figure}[t]
\centering
\includegraphics[width=0.95\columnwidth,trim={10 20 10 5},clip]{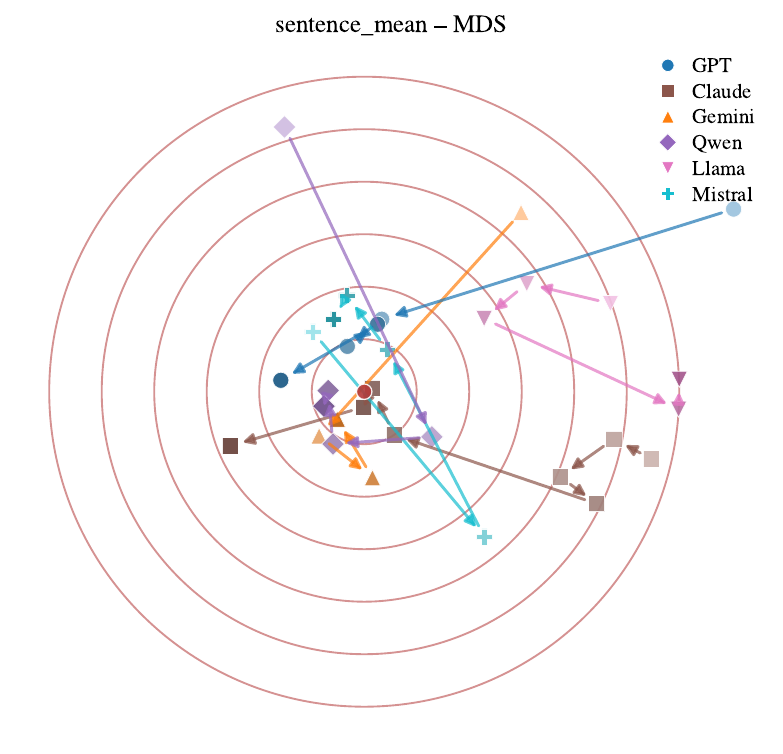}
\caption{Behavioral map obtained by applying metric MDS to the mean
per-prompt distance $D^{\mathrm{mean}}$. Marker shape and color identify model family; darker
shades denote later releases. Labels are omitted for readability and are
provided in Fig.~\ref{fig:mean_panels} and Table~\ref{tab:roster}.
\texttt{gpt-2} (the lightest GPT marker) is the global outlier.
\textcolor{ringred}{The concentric rings are centered on the dark red dot, the
mean map position of the six families' latest releases; their radii are evenly
spaced out to the $95$th percentile of distances to that center.}}
\label{fig:mean}
\end{figure}

\paragraph{Behavioral cartography (Fig.~\ref{fig:mean}).}
The mean per-prompt distance $D^{\mathrm{mean}}$ exhibits strong family structure: $21$ of the $32$
models have a nearest neighbor from the same family. The closest pair is
\texttt{gpt-3.5-turbo} and \texttt{gpt-4.1-mini}
($D^{\mathrm{mean}}\approx6.8$). In contrast, \texttt{gpt-2} has the largest
average distance to all other models ($\approx88.5$) and is therefore the global
outlier under this construction. Several cross-family nearest-neighbor
relations involve recent GPT, Gemini, and Qwen models, which is consistent with
the cross-family convergence examined below.

\paragraph{Relational geometry via GW (Fig.~\ref{fig:gw}, appendix).}
Because GW compares internal geometry rather than absolute embedding
coordinates, its family signal is weaker: $14/32$ models have a same-family
nearest neighbor. Several recent models pair across vendors, including
\texttt{gemini-3.1-pro-preview} with \texttt{qwen3.7-max} and an
\texttt{opus}-class Claude model with \texttt{gpt-5.2}. These pairings indicate
similar response-space geometry; they should not be interpreted as direct
comparisons of capability.

\paragraph{Drift (Fig.~\ref{fig:drift}).}
Family-wise drift quantifies the observed displacement between releases. The
\texttt{gpt-2}$\rightarrow$\texttt{gpt-3.5-turbo} transition is substantially
larger than every later step. Among releases dated 2024--2026, step distances
are smaller and broadly comparable across families. This is a descriptive
pattern; the distance alone does not identify which changes in data, training,
or decoding produced the drift.

\paragraph{Block and hierarchical structure (Fig.~\ref{fig:struct}).}
Ordering the distance matrix by family and date reveals low-distance
within-family blocks and two high-distance rows and columns:
\texttt{gpt-2} and \texttt{qwen-2.5-72b-instruct}. The average-linkage
dendrogram separates these models early; several subsequent clusters contain
models from multiple vendors. The hierarchy therefore reflects lineage only
partially, especially among recent releases.

\paragraph{Convergence and dispersion (Fig.~\ref{fig:conv}).}
Mean distance from each model to models in other families decreases with release
date, yielding a negative fitted trend. This provides descriptive evidence of
behavioral homogenization among later releases, although the observational
analysis does not establish a causal temporal mechanism. Response-cloud
dispersion provides a complementary within-model statistic. Several recent
reasoning-oriented models, including \texttt{claude-opus-4-7-think} and
\texttt{gpt-5.2}, have lower dispersion than some mid-generation predecessors;
under the chosen encoder, their responses therefore occupy a more compact
embedding cloud.

\begin{figure}[t]
\centering
\includegraphics[width=1\linewidth,trim={10 10 10 5},clip]{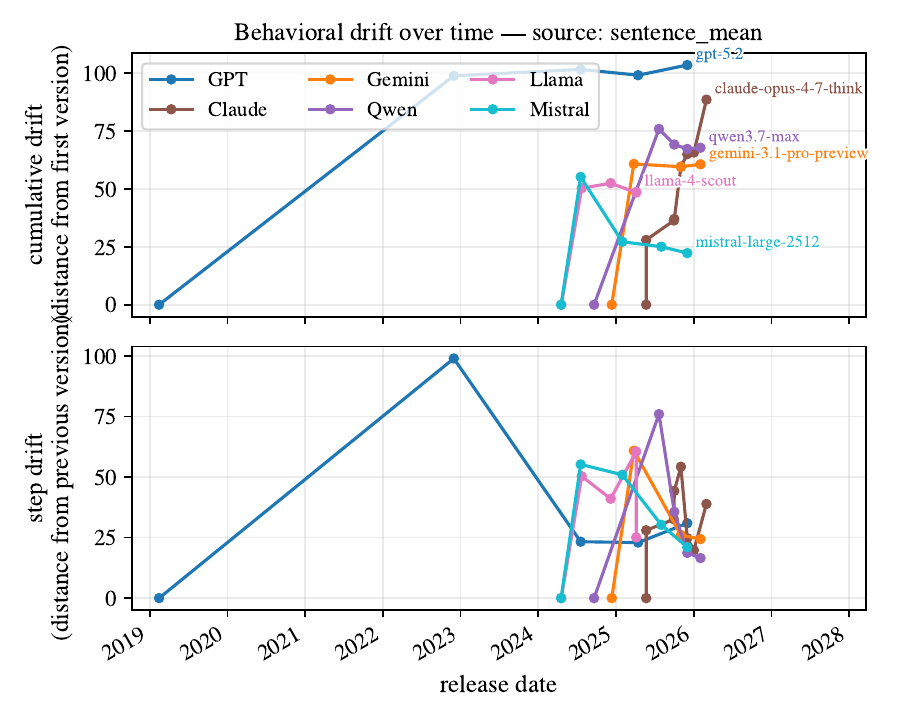}
\caption{Behavioral drift over time (mean metric). Top: cumulative drift from each family's first
version; bottom: step drift between consecutive versions. The \texttt{gpt-2}$\rightarrow$%
\texttt{gpt-3.5-turbo} jump dwarfs modern steps.}
\label{fig:drift}
\end{figure}

\begin{figure}[t]
\centering
\includegraphics[width=\linewidth]{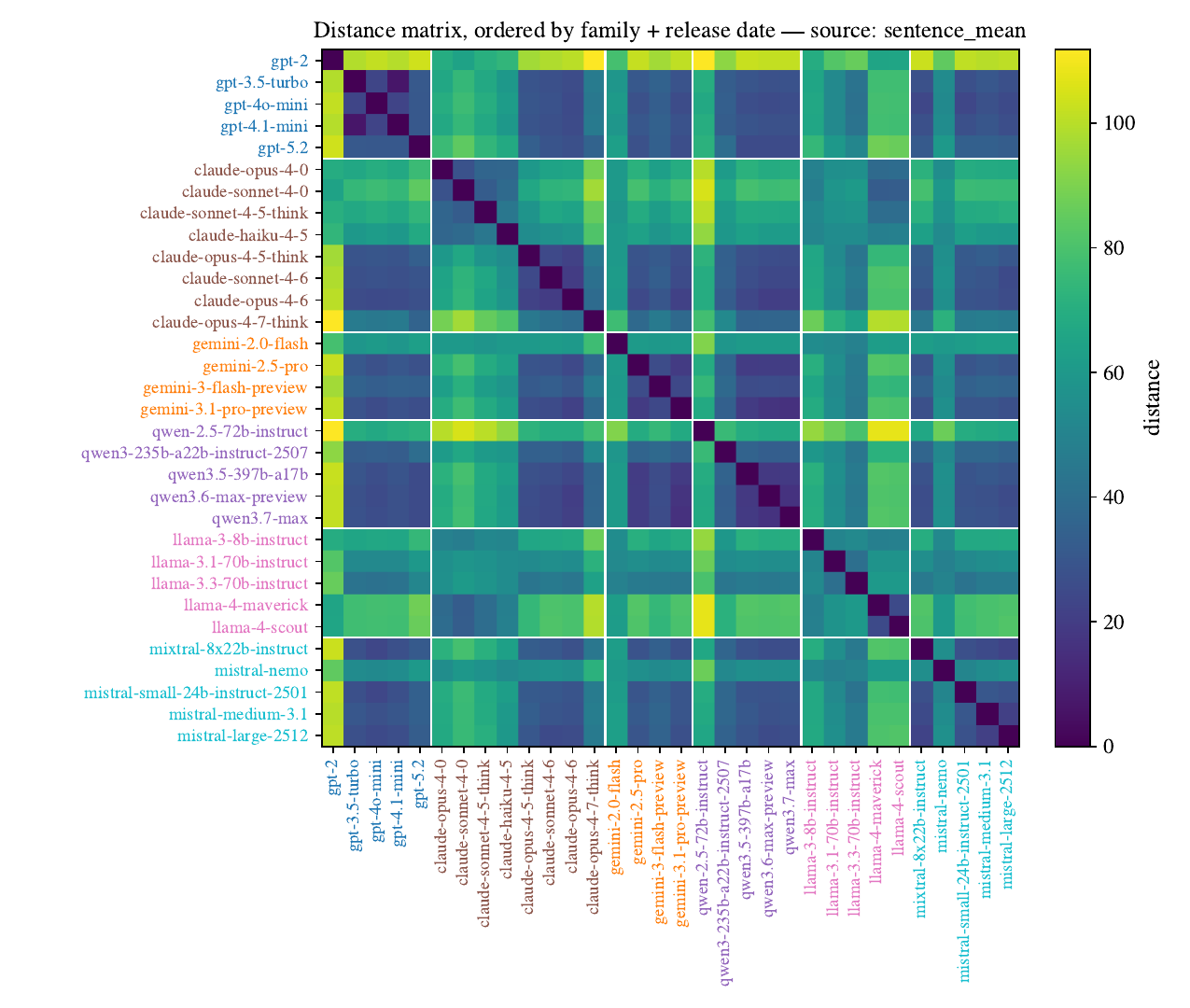}\\[4pt]
\includegraphics[width=\linewidth]{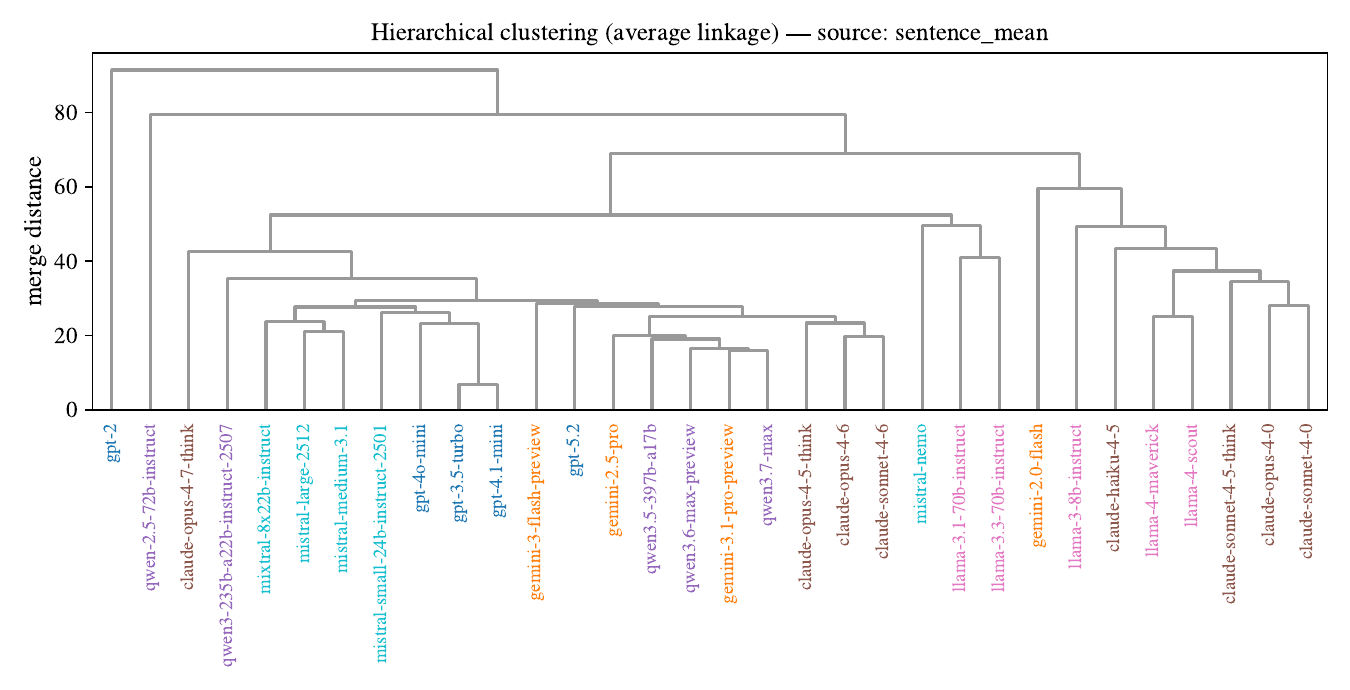}
\caption{Top: distance matrix ordered by family and release date; white lines separate families,
labels are family-colored. Bottom: average-linkage hierarchical clustering; \texttt{gpt-2} and
\texttt{qwen-2.5-72b-instruct} split off first, and recent models mix across families.}
\label{fig:struct}
\end{figure}

\begin{figure}[t]
\centering
\includegraphics[width=\linewidth]{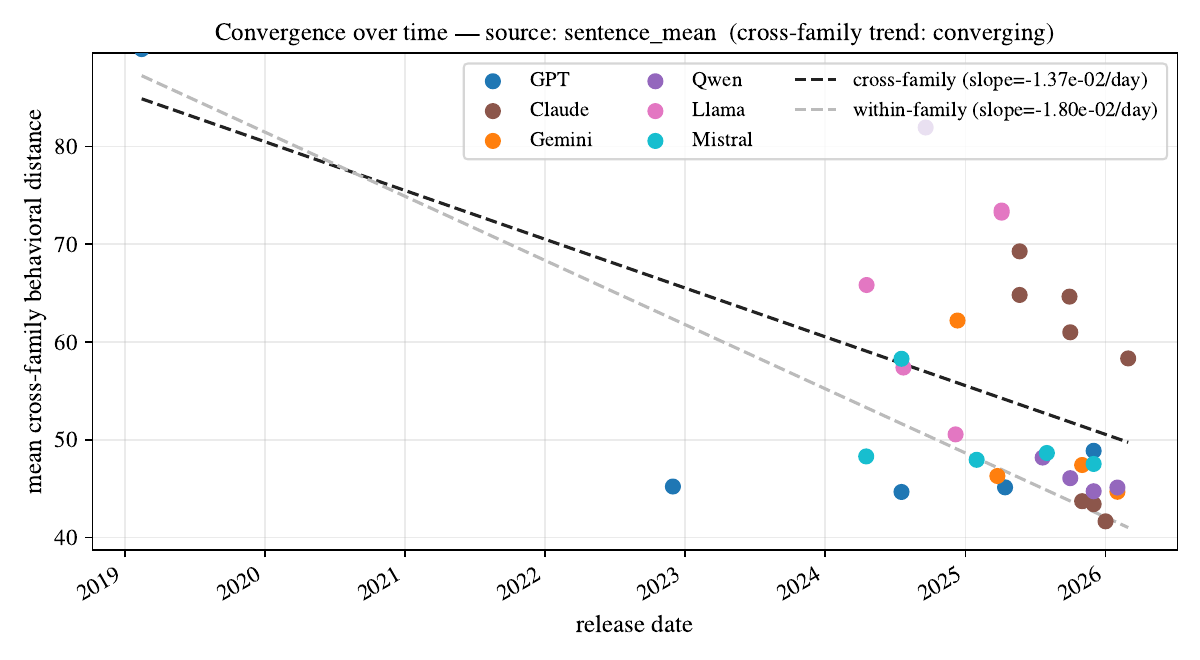}\\[4pt]
\includegraphics[width=\linewidth]{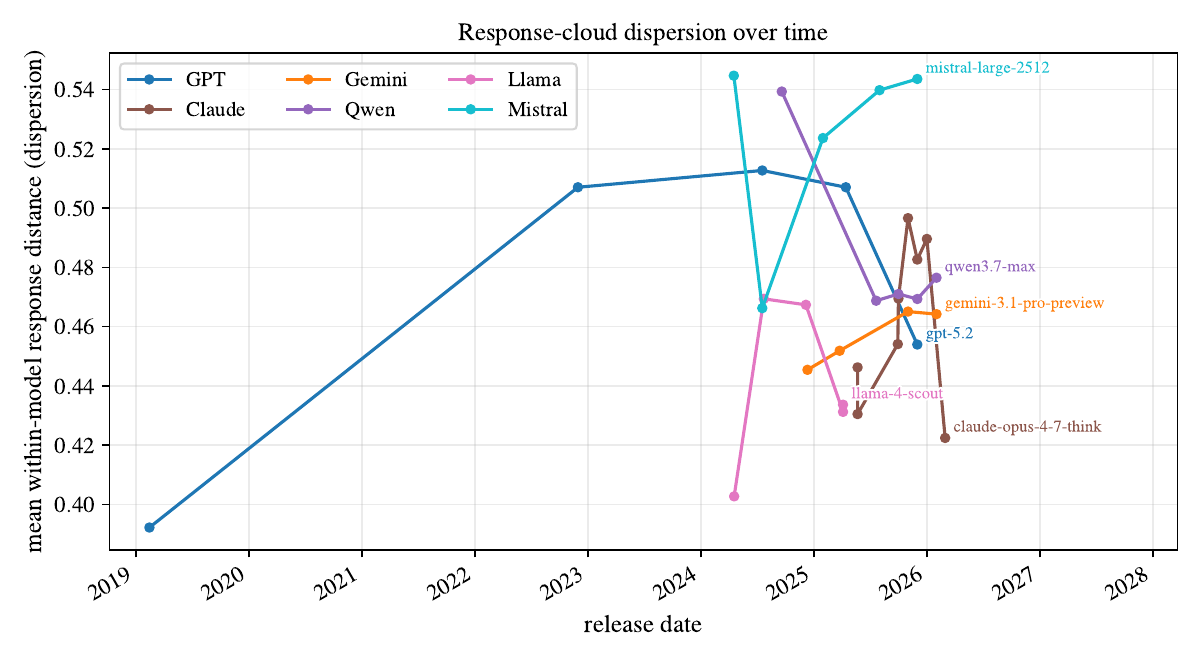}
\caption{Top: mean distance from each model to models in other families
versus release date, together with the fitted downward trend. Bottom:
within-model response-cloud dispersion over time. Several recent
reasoning-oriented models have comparatively compact response clouds.}
\label{fig:conv}
\end{figure}

\begin{figure*}[t]
\centering
\begin{minipage}{0.21\textwidth}\centering
\includegraphics[width=\linewidth]{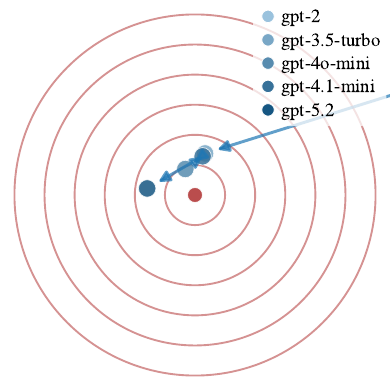}
\end{minipage}\hspace{0.02\textwidth}
\begin{minipage}{0.21\textwidth}\centering
\includegraphics[width=\linewidth]{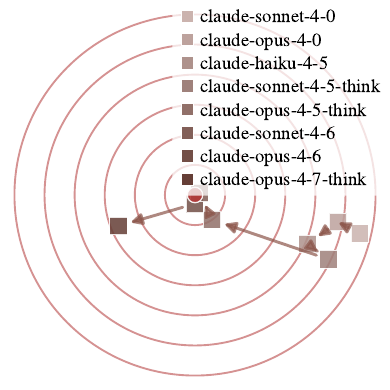}
\end{minipage}\hspace{0.02\textwidth}
\begin{minipage}{0.21\textwidth}\centering
\includegraphics[width=\linewidth]{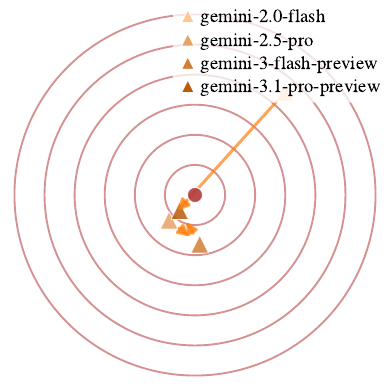}
\end{minipage}\hspace{0.02\textwidth}
\begin{minipage}{0.21\textwidth}\centering
\includegraphics[width=\linewidth]{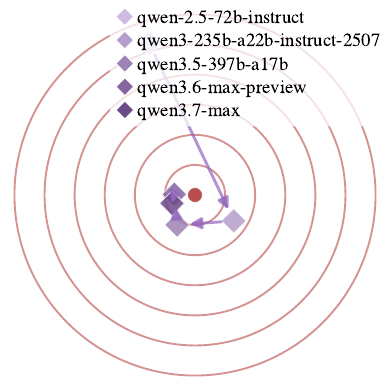}
\end{minipage}
\caption{Per-family detail for Fig.~\ref{fig:mean}: a common ring-bounded frame (identical axes)
with oldest$\rightarrow$latest arrows. Left to right: GPT, Claude, Gemini, Qwen; all six families
appear in Fig.~\ref{fig:mean_panels_full} (appendix).}
\label{fig:mean_panels}
\end{figure*}

\paragraph{Robustness across projections and constructions.}
The principal structure is stable across metric MDS, t-SNE, and UMAP
(Fig.~\ref{fig:proj}), so it is not specific to one projection. The
PCA-compressed score tracks the mean distance closely (Spearman $\rho=0.92$;
$17/32$ same-family nearest neighbors; Fig.~\ref{fig:pca}), while GW is less
correlated ($\rho=0.77$), as expected from its removal of fixed prompt
alignment and absolute coordinates. Family blocks, the \texttt{gpt-2} outlier,
and the downward trend nonetheless persist under all three sentence-level
constructions.

\paragraph{Robustness to the response encoder.}
The constructions above share one encoder, so their agreement probes
aggregation and alignment rather than the choice of $\phi$. Re-encoding all
$32\times10^{4}$ responses with \texttt{bge-m3},
\texttt{multilingual-e5-large}, and \texttt{all-mpnet-base-v2}---the last
$73\times$ smaller, none sharing a vendor with an evaluated model---preserves
the geometry: $\rho=0.92$--$0.93$ over the $496$ pairs, $19$--$22$ of $32$
same-family nearest neighbors against $21$, and a negative trend throughout
($r$ from $-0.52$ to $-0.55$). Every encoder returns \texttt{gpt-2} and
\texttt{qwen-2.5-72b-instruct} as the two outliers, so the Qwen-family outlier
is recovered by encoders with no Qwen lineage (Table~\ref{tab:encoders},
Fig.~\ref{fig:encoders}).

\begin{table}[t]
\centering
\small
\begin{tabular}{lcc}
\hline
Distance construction & same-family NN & $\rho$ vs.\ mean\\
\hline
Mean per-prompt (pseudometric) & $21/32$ & $1.00$\\
PCA-compressed           & $17/32$ & $0.92$\\
Gromov--Wasserstein      & $14/32$ & $0.77$\\
\hline
\end{tabular}
\caption{Agreement among sentence-level constructions. The first column
reports the fraction of models whose nearest neighbor is from the same family;
the second reports Spearman correlation over the $\binom{32}{2}=496$ distinct
model pairs.}
\label{tab:agree}
\end{table}

\section{A Token-level Lens: Per-prompt MMD}\label{sec_mmd}
The sentence-level analyses pool every response into one vector. As an
independent cross-check, we retain the token embeddings within each response and
compare the resulting empirical token measures prompt by prompt.

\paragraph{Construction.}
For two responses with token embeddings $X=\{x_a\}_{a=1}^{T}$ and
$Y=\{y_b\}_{b=1}^{S}$, define
$\widehat{\lambda}_X=T^{-1}\sum_{a=1}^{T}\delta_{x_a}$ and
$\widehat{\lambda}_Y=S^{-1}\sum_{b=1}^{S}\delta_{y_b}$. Their dissimilarity is
the squared MMD under an RBF kernel:
\begin{equation}
\begin{aligned}
\mathrm{MMD}^2(\widehat{\lambda}_X,\widehat{\lambda}_Y)
={}&\frac{1}{T^2}\sum_{a,a'}k(x_a,x_{a'})
+\frac{1}{S^2}\sum_{b,b'}k(y_b,y_{b'})\\
&-\frac{2}{TS}\sum_{a,b}k(x_a,y_b).
\end{aligned}
\label{eq:token-mmd}
\end{equation}
This is the exact squared RKHS distance between the two \emph{observed
empirical measures}; viewed as an estimator of a latent population-token MMD, it is the
usual biased plug-in estimator. It is well defined for every response length.
For two single-token responses, Eq.~\eqref{eq:token-mmd} reduces to
$2(1-\exp[-\|x-y\|_2^2/(2\sigma^2)])$, so no special case is needed. We use the
same estimator for all lengths because the unbiased U-statistic is undefined
for a one-token response. The plug-in estimator's finite-sample bias as an
estimator of population MMD depends on response length (median: $74$ tokens),
so unequal lengths could perturb comparisons. The high empirical agreement with
the sentence-level analysis reported below suggests that this issue does not
materially alter the model-pair rankings in our data.
For each prompt, Eq.~\eqref{eq:token-mmd} yields an $M\times M$ matrix. Stacking
over prompts gives an $N\times M\times M$ tensor. Token embeddings are streamed
rather than stored; only the approximately $40$ MB tensor is written to disk.
We aggregate it by the mean and by the same PCA procedure used for the
sentence-level profiles.

\paragraph{Agreement with the sentence-level analysis.}
The mean token-level matrix correlates strongly with the sentence-level mean
distance: Spearman $\rho=0.98$ and Pearson $0.97$ over the
$\binom{32}{2}=496$ distinct model pairs, and $28/32$ models share a nearest
neighbor under the two constructions. The token analysis recovers the same
qualitative structure: the \texttt{gpt-2} and \texttt{qwen-2.5-72b} outliers,
coherent family blocks (Figs.~\ref{fig:tok_map} and~\ref{fig:tok_struct}), and a
decreasing cross-family trend (Fig.~\ref{fig:tok_conv}). That a pooled sentence
representation and a token-distribution representation agree indicates the main
conclusions are not driven by one representation alone.

\begin{figure}[t]
\centering
\includegraphics[width=0.95\columnwidth,trim={10 20 10 5},clip]{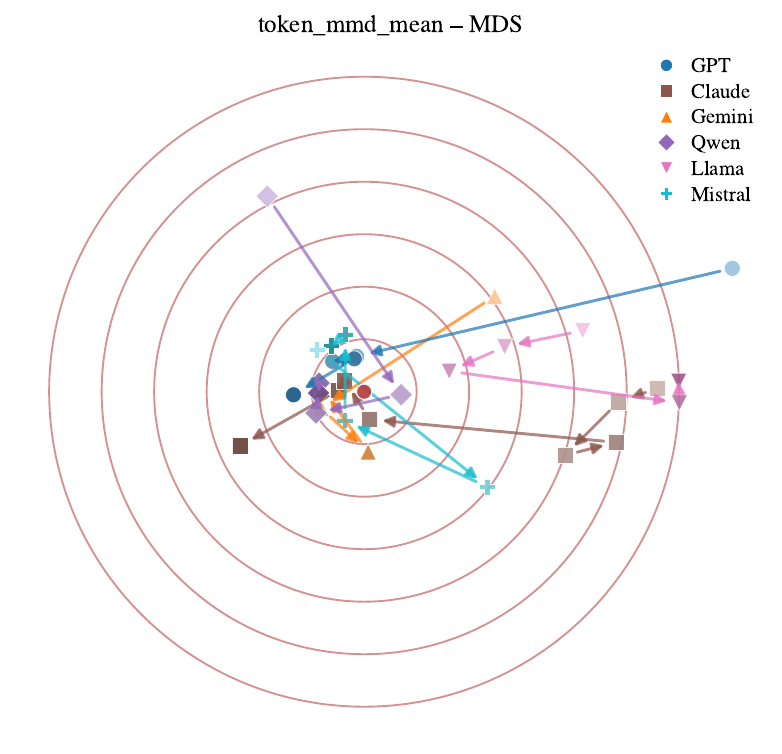}
\caption{Metric-MDS map of the token-level dissimilarity obtained by
averaging per-prompt squared MMD. Visual encoding follows Fig.~\ref{fig:mean}.
The family structure and the \texttt{gpt-2} outlier closely match the
sentence-level map (Spearman $\rho=0.98$); per-family details appear in
Fig.~\ref{fig:tok_panels}.}
\label{fig:tok_map}
\vspace{-10pt}
\end{figure}

\section{A Theory of Behavioral Similarity}
The empirical analysis is output-side; it does not identify the training causes
of convergence. We nevertheless give a sufficient training-side condition
under which two models must be behaviorally close.

Let $\mathcal{P}_m$ be an effective target population for model $m$, with prompt
marginal $\mathcal{P}_{m,X}$ and conditional response distribution
$p_m(r\mid x)$. For a model trained only by next-token prediction, this may be
interpreted as its training population. For a post-trained system, treating
$\mathcal{P}_m$ as an effective target distribution is an additional modeling
assumption. The deployed model defines $q_m(r\mid x)$, including the
end-of-sequence token, and has population negative-log-likelihood risk
\[
\mathcal{R}_m(q)=
\mathbb{E}_{(x,r)\sim\mathcal{P}_m}[-\log q(r\mid x)].
\]
The population minimizer over all conditional distributions is $p_m$, and the
excess risk is
\begin{equation}
\varepsilon_m
=\mathcal{R}_m(q_m)-\mathcal{R}_m(p_m)
=\mathbb{E}_{x\sim\mathcal{P}_{m,X}}
\mathrm{KL}\!\left(p_m(\cdot\mid x)\Vert q_m(\cdot\mid x)\right).
\label{eq:excess-risk}
\end{equation}

Let $\xi$ denote the inference-prompt distribution. For each
$j\in\{m,m'\}$, assume the covariate-coverage condition
\begin{equation}
\xi\ll\mathcal{P}_{j,X},
\qquad
\left\|\frac{d\xi}{d\mathcal{P}_{j,X}}\right\|_{\infty}
\leq\kappa_j<\infty.
\label{eq:coverage}
\end{equation}
Absolute continuity excludes inference regions that receive no mass under the
effective training prompt distribution, where training risk gives no control.
The density-ratio bound limits how strongly inference can upweight prompts that
are rare during training; $\kappa_j$ therefore quantifies the severity of
covariate shift. This assumption is weaker than equality of the training and
inference prompt distributions.

For probability measures $P$ and $Q$ on the response space, define
\[
\mathrm{TV}(P,Q)=\sup_{A}|P(A)-Q(A)|,
\]
where the supremum is over measurable events $A$. Total variation is the
largest probability discrepancy assigned to the same event.

\begin{theorem}[Behavioral similarity from population risk]
\label{thm:behavioral_similarity}
For models $m$ and $m'$, define the effective target-distribution discrepancy
\[
\eta^{\mathrm{txt}}_{m,m'}=
\mathbb{E}_{x\sim\xi}\mathrm{TV}\!\left(
 p_m(\cdot\mid x),p_{m'}(\cdot\mid x)\right).
\]
If the excess risks in Eq.~\eqref{eq:excess-risk} are finite and the coverage
condition in Eq.~\eqref{eq:coverage} holds for both models, then
\begin{equation}
\begin{aligned}
&\mathbb{E}_{x\sim\xi}\mathrm{TV}\!\left(
q_m(\cdot\mid x),q_{m'}(\cdot\mid x)\right)\\
&\quad\leq \eta^{\mathrm{txt}}_{m,m'}
+\sqrt{\frac{\kappa_m\varepsilon_m}{2}}
+\sqrt{\frac{\kappa_{m'}\varepsilon_{m'}}{2}}.
\end{aligned}
\label{eq:tv-behavior-bound}
\end{equation}
\end{theorem}
The bound separates deployed-model difference into one discrepancy between the
two effective targets and two learning-error terms. Conditional on these
quantities, the bound does not depend explicitly on architecture,
initialization, or optimizer. Small excess risk alone is insufficient: the
models can remain different when their effective target distributions differ.
The full proof appears in Appendix~\ref{sec:theory_training_behavior}.

Textually different responses may nevertheless be semantically close. Let
$\phi:\mathcal{R}\to\mathbb{R}^d$ be a semantic encoder, and define
$\nu_m^x=\phi_{\#}q_m(\cdot\mid x)$ and
$\pi_m^x=\phi_{\#}p_m(\cdot\mid x)$. For probability measures $\alpha$ and
$\beta$ on $\mathbb{R}^d$ with finite first moments, the 1-Wasserstein distance
with Euclidean ground cost is
\begin{equation}
W_1(\alpha,\beta)=
\inf_{\gamma\in\Pi(\alpha,\beta)}
\int\|u-v\|_2\,d\gamma(u,v),
\label{eq:w1-definition}
\end{equation}
where $\Pi(\alpha,\beta)$ is the set of couplings with marginals $\alpha$ and
$\beta$. Thus, $W_1$ is the smallest expected embedding distance achievable by
jointly aligning the two distributions.

Assume there exists a set $\mathcal{S}_\phi\subset\mathbb{R}^d$ containing the
supports of $\nu_j^x$ and $\pi_j^x$ for $j\in\{m,m'\}$ and $\xi$-almost every
$x$, with
$\operatorname{diam}(\mathcal{S}_\phi)\leq B_\phi$. Define
\[
\eta^{\mathrm{sem}}_{m,m'}=
\mathbb{E}_{x\sim\xi}W_1(\pi_m^x,\pi_{m'}^x).
\]

\begin{corollary}[Semantic behavioral similarity]
\label{cor:semantic_similarity}
Under the assumptions of Theorem~\ref{thm:behavioral_similarity} and the bounded
embedding-diameter condition above,
\begin{equation}
\begin{aligned}
&\mathbb{E}_{x\sim\xi}W_1(\nu_m^x,\nu_{m'}^x)\\
&\quad\leq \eta^{\mathrm{sem}}_{m,m'}
+B_\phi\sqrt{\frac{\kappa_m\varepsilon_m}{2}}
+B_\phi\sqrt{\frac{\kappa_{m'}\varepsilon_{m'}}{2}}.
\end{aligned}
\label{eq:semantic-behavior-bound}
\end{equation}
\end{corollary}

\begin{corollary}[Conditional mean-embedding similarity]
\label{cor:mean_embedding_similarity}
Let
\[
\mu_m(x)=\mathbb{E}_{r\sim q_m(\cdot\mid x)}[\phi(r)]
\]
be the conditional mean response embedding. Under the assumptions of
Corollary~\ref{cor:semantic_similarity},
\begin{equation}
\begin{aligned}
&\mathbb{E}_{x\sim\xi}
\|\mu_m(x)-\mu_{m'}(x)\|_2\\
&\quad\leq \eta^{\mathrm{sem}}_{m,m'}
+B_\phi\sqrt{\frac{\kappa_m\varepsilon_m}{2}}
+B_\phi\sqrt{\frac{\kappa_{m'}\varepsilon_{m'}}{2}}.
\end{aligned}
\label{eq:mean-embedding-bound}
\end{equation}
\end{corollary}
Corollary~\ref{cor:mean_embedding_similarity} controls the distance between
conditional mean embeddings. The empirical distance in Eq.~\eqref{eq:dmean}
instead uses one generated response from each model and prompt. Its expectation
is bounded by the same right-hand side plus one within-prompt spread term for
each model; Appendix~\ref{sec:theory_training_behavior} states this connection
precisely. Consequently, decreasing excess risks and increasingly similar
effective target distributions are sufficient for relative semantic
homogenization, potentially toward a nonzero floor. These are distributional
results. A guarantee for deterministic greedy decoding requires an additional
stability condition, such as a positive probability margin.

\section{Tree-likeness and Low-rank Structure}
\paragraph{Approximate tree structure.}
We quantify tree-likeness using the four-point defect associated with
Gromov $\delta$-hyperbolicity (definition and qualification in
Appendix~\ref{app:results}). After normalization by matrix diameter, the mean
defect is $0.015$ for both the sentence and token constructions. The maximum
normalized defect is larger ($0.12$--$0.16$), indicating small average but
localized departures from a tree-like geometry. Because the token matrix
averages squared MMD, its four-point value is a descriptive diagnostic rather
than a formal hyperbolicity constant. These localized defects are consistent
with cross-family mixing in the dendrogram but do not establish a particular
evolutionary mechanism.

\paragraph{Low-dimensional disagreement structure.}
The prompt-wise token-MMD tensor admits a compact low-rank approximation: its
leading component explains $63\%$ of prompt-to-prompt variance, and the first
five explain $77\%$. A few axes thus summarize much of how model pairs disagree
across prompts, supporting the PCA summary and suggesting that high-leverage
prompts could yield smaller diagnostic sets, though prompt selection is not
evaluated here.

\section{Discussion and Limitations}
The agreement across constructions indicates that the family structure and decreasing cross-family distances are not specific to one representation or alignment assumption. Behavioral similarity does not, however, imply correctness. Absolute distances depend on the prompt bank, encoder, and decoding configuration, though the reported structure survives the encoder swaps above. With one response per model--prompt pair, the distances also carry generation variability; release-date trends are observational, GW uses a 256-prompt subsample, and the model panel is not exhaustive. Finally, the training-side quantities are unobserved here, so the theory supplies a sufficient scenario rather than an empirical explanation.

\section{Conclusion}
We presented a label-free framework for mapping LLM behavior from responses to a shared prompt bank. Across sentence- and token-level constructions it reveals family coherence, outliers, decreasing cross-family distances, and compact response clouds for recent models. We also gave an architecture-agnostic sufficient condition for behavioral similarity that, with its training-side quantities unobserved, complements rather than explains the findings. Output geometry augments leaderboards when weights, activations, or labels are unavailable.

\bibliography{references}

\onecolumn

\appendix

\section{Supplementary Material}
\label{app:supp}

\subsection{Additional Proofs and Derivations}
\label{app:proofs}

\paragraph{Full proof of Proposition~\ref{prop:pseudometric}.}

\begin{proposition}
\label{prop:pseudometric}
The mean per-prompt distance in Eq.~\eqref{eq:dmean} is a pseudometric on the
observed model representations $\{E_m\}_{m=1}^{M}$ and a metric on the quotient
space induced by equality of all encoded responses.
\end{proposition}
The proof is given in Appendix~\ref{app:proofs}.

For each prompt $i$, define
$d_i(m,m')=\|E_m[i]-E_{m'}[i]\|_2$. Then
$D^{\mathrm{mean}}_{m,m'}=N^{-1}\sum_{i=1}^{N}d_i(m,m')$.

\begin{proof}
For every prompt $i$, Euclidean distance implies nonnegativity, symmetry, and
\[
d_i(m,m')\leq d_i(m,\ell)+d_i(\ell,m')
\]
for any third model $\ell$. Averaging these pointwise inequalities over
$i=1,\ldots,N$ preserves all three properties, and
$D^{\mathrm{mean}}_{m,m}=0$. Hence $D^{\mathrm{mean}}$ is a pseudometric.

It may fail identity of indiscernibles at the model level. Specifically,
$D^{\mathrm{mean}}_{m,m'}=0$ if and only if
$E_m[i]=E_{m'}[i]$ for every prompt $i$. Distinct underlying models can satisfy
this equality if they produce encoder-equivalent responses on the evaluated
prompts. After quotienting by the relation
$m\sim m'\iff E_m[i]=E_{m'}[i]$ for all $i$, identity of indiscernibles holds,
so the induced function is a metric on the quotient space.
\end{proof}

\paragraph{Token MMD for single-token responses.}
Let
$\widehat{\lambda}_X=T^{-1}\sum_{a=1}^{T}\delta_{x_a}$ and
$\widehat{\lambda}_Y=S^{-1}\sum_{b=1}^{S}\delta_{y_b}$ be the empirical token
measures of two responses. The plug-in squared MMD is
\begin{equation}
\widehat{\mathrm{MMD}}^2(X,Y)=
\frac{1}{T^2}\sum_{a,a'}k(x_a,x_{a'})
+\frac{1}{S^2}\sum_{b,b'}k(y_b,y_{b'})
-\frac{2}{TS}\sum_{a,b}k(x_a,y_b).
\label{eq:plugin}
\end{equation}
For single-token responses, $T=S=1$ and $X=\{x\}$, $Y=\{y\}$. With the RBF
kernel $k(u,v)=\exp(-\|u-v\|_2^2/(2\sigma^2))$, Eq.~\eqref{eq:plugin} becomes
\begin{equation}
\widehat{\mathrm{MMD}}^2(X,Y)
=2\left(1-e^{-\|x-y\|_2^2/(2\sigma^2)}\right).
\label{eq:single}
\end{equation}
It is zero when $x=y$, increases monotonically with $\|x-y\|_2$, and approaches
$2$ as the distance diverges. Thus, a one-token response is handled naturally
as a one-point empirical measure. Equation~\eqref{eq:plugin} is exact for the
two observed empirical measures; if interpreted as an estimator of an
underlying population-token MMD, it remains the biased plug-in estimator, with
bias that can be substantial for short responses.

\paragraph{Relationships among the comparison constructions.}
The constructions are related, but they are not all metrics of the same form.
The mean sentence-level construction averages Euclidean distances under a fixed
prompt alignment and is a pseudometric by
Proposition~\ref{prop:pseudometric}. With a characteristic kernel, MMD is a
metric on probability measures \citep{gretton2012mmd}; our token entries use
its square, which is a dissimilarity whose square root is the corresponding
metric. The PCA construction is instead a rank-one compression of the vector of
aligned prompt-wise distances and need not satisfy metric axioms. GW replaces
the fixed prompt alignment by an optimized coupling between internal cost
matrices. When the within-space costs are metrics, the square root of the standard
GW objective is a metric on metric-measure spaces modulo measure-preserving
isometry \citep{memoli2011gw}. Our within-model costs are cosine
dissimilarities, which need not satisfy the triangle inequality; the appropriate
object is therefore a measure-network discrepancy rather than a formal metric.
These costs are nevertheless invariant to orthogonal transformations and
positive rescaling of the embedding vectors.

Accordingly, the four analyses vary along three axes: representation (pooled
response embeddings or token empirical measures), aggregation (mean or PCA),
and alignment (fixed prompts or optimized coupling). The empirical correlations
in Table~\ref{tab:agree-full} show that the resulting rankings are similar, but
they do not imply mathematical equivalence of the constructions.

\subsection{Additional Experimental Details}
\label{app:details}

\paragraph{Model roster.}
Table~\ref{tab:roster} lists all 32 models with the release dates used as the
time axis; the figures abbreviate these identifiers. Dates marked
$^{*}$ are approximate (post-cutoff or unannounced at the time of the runs);
they affect only the horizontal placement of points, never the distances.

\begin{table*}[t]
\centering
\footnotesize
\begin{tabular}{lll}
\hline
Family & Model ID & Release date\\
\hline
GPT & \texttt{gpt-2} & 2019-02-14\\
 & \texttt{gpt-3.5-turbo} & 2022-11-30\\
 & \texttt{gpt-4o-mini} & 2024-07-18\\
 & \texttt{gpt-4.1-mini} & 2025-04-14\\
 & \texttt{gpt-5.2} & 2025-12-01$^{*}$\\
\hline
Claude & \texttt{claude-sonnet-4-0} & 2025-05-22\\
 & \texttt{claude-opus-4-0} & 2025-05-22\\
 & \texttt{claude-sonnet-4-5-think} & 2025-09-29$^{*}$\\
 & \texttt{claude-haiku-4-5} & 2025-10-01$^{*}$\\
 & \texttt{claude-opus-4-5-think} & 2025-11-01$^{*}$\\
 & \texttt{claude-sonnet-4-6} & 2025-12-01$^{*}$\\
 & \texttt{claude-opus-4-6} & 2026-01-01$^{*}$\\
 & \texttt{claude-opus-4-7-think} & 2026-03-01$^{*}$\\
\hline
Gemini & \texttt{gemini-2.0-flash} & 2024-12-11\\
 & \texttt{gemini-2.5-pro} & 2025-03-25\\
 & \texttt{gemini-3-flash-preview} & 2025-11-01$^{*}$\\
 & \texttt{gemini-3.1-pro-preview} & 2026-02-01$^{*}$\\
\hline
Qwen & \texttt{qwen/qwen-2.5-72b-instruct} & 2024-09-19\\
 & \texttt{qwen3-235b-a22b-instruct-2507} & 2025-07-21\\
 & \texttt{qwen/qwen3.5-397b-a17b} & 2025-10-01$^{*}$\\
 & \texttt{qwen/qwen3.6-max-preview} & 2025-12-01$^{*}$\\
 & \texttt{qwen3.7-max} & 2026-02-01$^{*}$\\
\hline
Llama & \texttt{meta-llama/llama-3-8b-instruct} & 2024-04-18\\
 & \texttt{meta-llama/llama-3.1-70b-instruct} & 2024-07-23\\
 & \texttt{meta-llama/llama-3.3-70b-instruct} & 2024-12-06\\
 & \texttt{meta-llama/llama-4-scout} & 2025-04-05\\
 & \texttt{meta-llama/llama-4-maverick} & 2025-04-05\\
\hline
Mistral & \texttt{mistralai/mixtral-8x22b-instruct} & 2024-04-17\\
 & \texttt{mistralai/mistral-nemo} & 2024-07-18\\
 & \texttt{mistralai/mistral-small-24b-instruct-2501} & 2025-01-30\\
 & \texttt{mistralai/mistral-medium-3.1} & 2025-08-01$^{*}$\\
 & \texttt{mistralai/mistral-large-2512} & 2025-12-01$^{*}$\\
\hline
\end{tabular}
\caption{The full 32-model roster: model identifier and the release date used
on the time axis. Figures abbreviate these identifiers. $^{*}$:
approximate date (post-cutoff or unannounced at the time of the runs);
approximate dates affect only horizontal placement, never the distances.}
\label{tab:roster}
\end{table*}

\paragraph{Question bank composition.}
Every model answers the same canonical bank of $N{=}10{,}000$ prompts, drawn as
a stratified sample from 13 public benchmarks: MMLU \citep{hendrycks2021mmlu},
ARC \citep{clark2018arc}, BIG-Bench \citep{srivastava2023beyond}, HellaSwag
\citep{zellers2019hellaswag}, WinoGrande \citep{sakaguchi2020winogrande}, MATH
\citep{hendrycks2021math}, GSM8K \citep{cobbe2021gsm8k}, SQuAD~2.0
\citep{rajpurkar2018squad2}, TriviaQA \citep{joshi2017triviaqa}, Natural
Questions \citep{kwiatkowski2019nq}, TruthfulQA \citep{lin2022truthfulqa},
HumanEval \citep{chen2021humaneval}, and MBPP \citep{austin2021mbpp}.
Table~\ref{tab:benchmarks} gives the per-benchmark counts. The multiple-choice and short-answer benchmarks are
sampled to a common target ($\approx$876 each); the two code benchmarks
(\texttt{HumanEval}, \texttt{MBPP}) contribute their smaller available size.
Response index $i$ refers to the same prompt for every model. The bank stores
prompts only, with no gold labels, consistent with the label-free design.

\begin{table}[t]
\centering
\small
\begin{tabular}{lr}
\hline
Benchmark & \# prompts\\
\hline
BIG-Bench          & 877\\
ARC                & 877\\
MMLU               & 876\\
HellaSwag          & 876\\
WinoGrande         & 876\\
MATH               & 876\\
GSM8K              & 876\\
SQuAD 2.0          & 876\\
TriviaQA           & 876\\
Natural Questions  & 876\\
TruthfulQA         & 817\\
MBPP               & 257\\
HumanEval          & 164\\
\hline
Total              & 10{,}000\\
\hline
\end{tabular}
\caption{Composition of the canonical 10k question bank by source benchmark.}
\label{tab:benchmarks}
\end{table}

\paragraph{Response embeddings.}
Responses are encoded with \texttt{Qwen3-Embedding-8B}. Each response is read as
at most $1024$ tokens (truncation limit on the input \emph{sequence length});
the encoder returns one hidden vector per token in its native $4096$-dimensional
output space (the model's hidden width, fixed by architecture and independent of
response length), and we mean-pool over the token positions into a single
$4096$-dimensional vector per response. Thus a length-$T$ response ($T\le 1024$)
maps a $(T,4096)$ array of token embeddings to one $4096$-vector; the two numbers
are different axes---$1024$ counts input tokens, $4096$ is the output width. We
use no $L_2$ normalization and an encoder batch size of $8$, giving one matrix
$E_m\in\mathbb{R}^{10000\times 4096}$ per model, row-aligned across models by
prompt.

\paragraph{Cross-model dissimilarities.}
The mean per-prompt distance and PCA-compressed scores are computed directly from
$\{E_m\}$ using Euclidean per-prompt distances. For Gromov--Wasserstein we build,
per model, a cosine response-by-response cost matrix over a common seeded
subsample of $256$ responses (seed $0$), assign uniform marginals, and solve
GW$^2$ with the conditional-gradient solver of POT. The token-level MMD uses the
same encoder with at most $256$ tokens per response, an RBF kernel with a
global-median bandwidth (estimated from a sample of responses), and the biased
(plug-in) estimator applied uniformly to all response lengths; the per-prompt
$32\times 32$ matrices are then averaged over prompts.

\paragraph{Projections.}
2D embeddings are computed from each precomputed $32\times 32$ distance matrix
with metric MDS (SMACOF, \texttt{max\_iter}$=300$, \texttt{n\_init}$=4$), t-SNE
(\texttt{perplexity}$=5$), and UMAP (\texttt{n\_neighbors}$=5$,
\texttt{min\_dist}$=0.1$). All three use \texttt{metric=precomputed} and seed
$0$. The per-family panels reuse the overview's 2D coordinates and share one ring
center.

\paragraph{Software and compute.}
Response generation is routed through OpenRouter, except the local \texttt{gpt-2}
baseline, which runs on-device; each model produces one response per prompt under
a single decoding configuration. Embeddings are produced on a single local GPU,
while all distances, projections, and fits run on CPU with NumPy, scikit-learn,
and POT. Per-response embeddings are cached, so any distance can be recomputed
without re-encoding.

\subsection{Additional Results}
\label{app:results}

\paragraph{Full agreement among distance constructions.}
Table~\ref{tab:agree} in the main text reports each construction's Spearman
correlation with the mean metric. Table~\ref{tab:agree-full} gives the complete
pairwise matrix over the $\binom{32}{2}{=}496$ distinct off-diagonal model pairs.
Beyond the mean-metric column (which reproduces Table~\ref{tab:agree}), two
patterns stand out: the token-level MMD is the closest match to the sentence mean
metric ($\rho{=}0.98$) and also tracks the PCA-compressed variant tightly
($\rho{=}0.95$), while Gromov--Wasserstein is the most distinct construction,
correlating no higher than $0.83$ with any other. This is consistent with GW discarding fixed prompt alignment and absolute
embedding coordinates, and therefore capturing a different aspect of response
geometry.

\begin{table}[t]
\centering
\small
\begin{tabular}{lcccc}
\hline
 & Mean & PCA & GW & Token\\
\hline
Mean  & $1.00$ & $.92$  & $.77$  & $.98$\\
PCA   &        & $1.00$ & $.80$  & $.95$\\
GW    &        &        & $1.00$ & $.83$\\
Token &        &        &        & $1.00$\\
\hline
\end{tabular}
\caption{Spearman correlation between distance constructions over the $496$
distinct off-diagonal model pairs (upper triangle). The Mean column matches
Table~\ref{tab:agree}.}
\label{tab:agree-full}
\end{table}

\paragraph{Structure tests.}
For a quadruple $(a,b,c,d)$, let $S_1,S_2,S_3$ be the three sums
$D_{ab}+D_{cd}$, $D_{ac}+D_{bd}$, and $D_{ad}+D_{bc}$, ordered so that
$S_{(1)}\leq S_{(2)}\leq S_{(3)}$. We define its four-point defect as
\[
\delta(a,b,c,d)=\frac{S_{(3)}-S_{(2)}}{2}.
\]
For a finite metric, every defect is zero exactly when the metric is
$0$-hyperbolic (equivalently, tree-additive on the observed points). Small
normalized defects therefore quantify approximate tree-likeness when the input
is a metric. We report the mean and maximum defect over evaluated quadruples,
divided by matrix diameter. The sentence mean is a metric on its quotient space;
the token matrix averages squared MMD values and need not satisfy the triangle
inequality, so its defect is interpreted only as a descriptive tree-likeness
diagnostic. Table~\ref{tab:hyp} shows a small mean normalized defect
($\approx0.015$) for both constructions and a larger maximum
($0.12$--$0.16$), indicating localized departures from tree structure.

The prompt-wise token-MMD tensor also admits a compact low-rank approximation:
the leading component explains $63\%$ of prompt-to-prompt variance, increasing
to $69\%$, $72\%$, $77\%$, and $85\%$ at ranks $2$, $3$, $5$, and $10$,
respectively. The rank-one approximation is the PCA summary used in the main
text.

\begin{table}[t]
\centering
\small
\begin{tabular}{lcc}
\hline
 & Sentence mean & Token MMD\\
\hline
$\bar\delta/\text{diam}$ (mean rel.)      & $0.015$ & $0.015$\\
$\delta_{\max}/\text{diam}$ (max rel.)    & $0.125$ & $0.155$\\
diameter                                  & $111.8$ & $0.317$\\
\hline
\end{tabular}
\caption{Normalized four-point defects. For the sentence mean, these are hyperbolicity statistics on the quotient metric; for the averaged squared-MMD matrix, they are descriptive tree-likeness diagnostics.}
\label{tab:hyp}
\end{table}

\paragraph{Encoder robustness.}
All constructions in the main text share one response encoder, so their mutual
agreement tests aggregation and alignment but not the choice of $\phi$. We
therefore re-encoded all $32\times10^{4}$ responses with three additional
encoders and recomputed $D^{\mathrm{mean}}$ end to end. The alternatives span
architecture, scale, and provenance: \texttt{bge-m3} and
\texttt{multilingual-e5-large} are XLM-RoBERTa encoders ($568$M and $560$M
parameters, $1024$ dimensions), and \texttt{all-mpnet-base-v2} is a $110$M
MPNet sentence encoder ($768$ dimensions), roughly $73\times$ smaller than the
$8$B, $4096$-dimensional Qwen3-Embedding-8B. None of the three shares a vendor
with any evaluated model, which also addresses the concern that a Qwen encoder
might favor Qwen models.

Table~\ref{tab:encoders} reports the outcome. Rank agreement with the reference
encoder is high ($\rho=0.92$--$0.93$ over the $496$ distinct pairs), the
same-family nearest-neighbor count is essentially unchanged ($19$--$22$ of $32$
against $21$), and \emph{every} encoder returns \texttt{gpt-2} and
\texttt{qwen-2.5-72b-instruct} as the two global outliers---so the Qwen-family
outlier is recovered by encoders with no Qwen lineage. The convergence trend is
likewise stable: its slope changes with the absolute distance scale, which is
encoder-specific, but its correlation with release date is nearly invariant
($r$ between $-0.52$ and $-0.55$, one-sided $p\le0.0011$ throughout).
Figure~\ref{fig:encoders} shows the corresponding maps. Because metric MDS is
determined only up to rotation, reflection, and scale, each map is
Procrustes-aligned to the reference before plotting; the aligned
configurations are visually near-identical.

Two honest qualifications. First, \texttt{multilingual-e5-large} and
\texttt{all-mpnet-base-v2} accept at most $512$ tokens rather than $1024$;
sampling $2{,}400$ responses gives a median length of $80$ tokens with $4.5\%$
above $512$, so a small minority of long responses is truncated more
aggressively, which can only depress agreement. \texttt{bge-m3} runs at the
reference length of $1024$ and is thus the exactly matched comparison. Second,
the three alternatives agree with one another ($\rho=0.99$) more closely than
any of them agrees with Qwen3-Embedding-8B, which is plausibly the one
decoder-based, $4096$-dimensional encoder in the panel; the reported findings
nonetheless hold under all four.

\begin{table}[t]
\centering
\small
\begin{tabular}{lccccc}
\hline
Encoder & Dim & Len & $\rho$ & same-fam.\ NN & trend $r$\\
\hline
Qwen3-Embedding-8B (ref.)      & $4096$ & $1024$ & ---     & $21/32$ & $-0.522$\\
\texttt{bge-m3}                & $1024$ & $1024$ & $0.933$ & $21/32$ & $-0.521$\\
\texttt{multilingual-e5-large} & $1024$ & $512$  & $0.922$ & $22/32$ & $-0.551$\\
\texttt{all-mpnet-base-v2}     & $768$  & $512$  & $0.930$ & $19/32$ & $-0.522$\\
\hline
\end{tabular}
\caption{Encoder robustness. $\rho$ is the Spearman correlation of
$D^{\mathrm{mean}}$ with the reference encoder over the $496$ distinct model
pairs; ``same-fam.\ NN'' counts models whose nearest neighbor is from their own
family; ``trend $r$'' is the correlation between mean cross-family distance and
release date (negative $=$ converging; one-sided $p\le0.0011$ in every row).
All four encoders return \texttt{gpt-2} and \texttt{qwen-2.5-72b-instruct} as
the two global outliers.}
\label{tab:encoders}
\end{table}

\begin{figure*}[t]
\centering
\includegraphics[width=\textwidth]{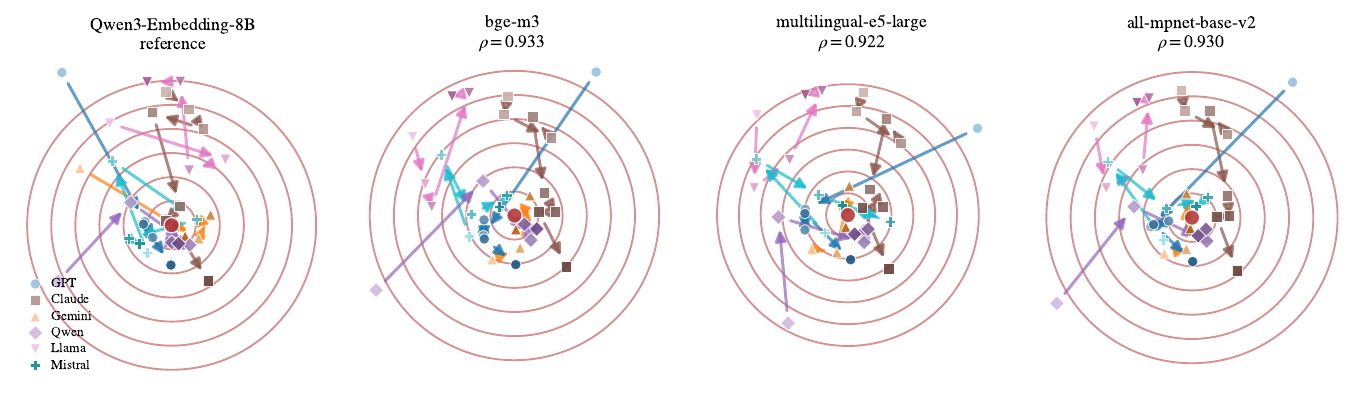}
\caption{Metric-MDS maps of $D^{\mathrm{mean}}$ under four response encoders,
drawn with the encoding of Fig.~\ref{fig:mean}: family markers shaded
light-to-dark by release order, oldest$\rightarrow$latest arrows, and
concentric rings about the latest-models centroid. Each map is
Procrustes-aligned to the reference (left), since MDS is determined only up to
rotation, reflection, and scale; the rings are recomputed per panel from that
encoder's own coordinates, so their agreement reflects the data rather than a
reused reference geometry. The configurations are near-identical---the
\texttt{gpt-2} outlier, the family blocks, and the inward evolution
paths---under an $8$B decoder-based encoder and a $110$M MPNet encoder alike.}
\label{fig:encoders}
\end{figure*}

\paragraph{Supplementary figures.}
Figure~\ref{fig:mean_panels_full} completes Fig.~\ref{fig:mean_panels} with the Llama and Mistral
panels; Fig.~\ref{fig:gw} maps the ecosystem under the alignment-invariant GW distance;
Figs.~\ref{fig:gw_panels} and \ref{fig:tok_panels} give the per-family detail for the GW and
token-level cartography maps; Figs.~\ref{fig:proj} and \ref{fig:pca}
show the projection and PCA robustness maps; and Figs.~\ref{fig:tok_struct} and \ref{fig:tok_conv}
the token-level structure and convergence.

\begin{figure*}[t]
\centering
\begin{minipage}{0.26\textwidth}\centering
\includegraphics[width=\linewidth]{map_sentence_mean_panel_gpt.pdf}
\end{minipage}\hspace{0.02\textwidth}
\begin{minipage}{0.26\textwidth}\centering
\includegraphics[width=\linewidth]{map_sentence_mean_panel_claude.pdf}
\end{minipage}\hspace{0.02\textwidth}
\begin{minipage}{0.26\textwidth}\centering
\includegraphics[width=\linewidth]{map_sentence_mean_panel_gemini.pdf}
\end{minipage}\\[3pt]
\begin{minipage}{0.26\textwidth}\centering
\includegraphics[width=\linewidth]{map_sentence_mean_panel_qwen.pdf}
\end{minipage}\hspace{0.02\textwidth}
\begin{minipage}{0.26\textwidth}\centering
\includegraphics[width=\linewidth]{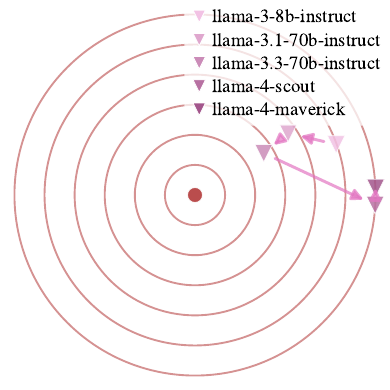}
\end{minipage}\hspace{0.02\textwidth}
\begin{minipage}{0.26\textwidth}\centering
\includegraphics[width=\linewidth]{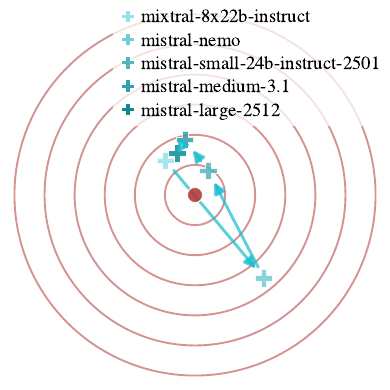}
\end{minipage}
\caption{All six per-family panels for Fig.~\ref{fig:mean} in a common ring-bounded frame
(identical axes) with oldest$\rightarrow$latest arrows. Left to right, top to bottom: GPT,
Claude, Gemini, Qwen, Llama, Mistral. Fig.~\ref{fig:mean_panels} shows the first four in the
main text.}
\label{fig:mean_panels_full}
\end{figure*}

\begin{figure}[t]
\centering
\includegraphics[width=0.55\textwidth]{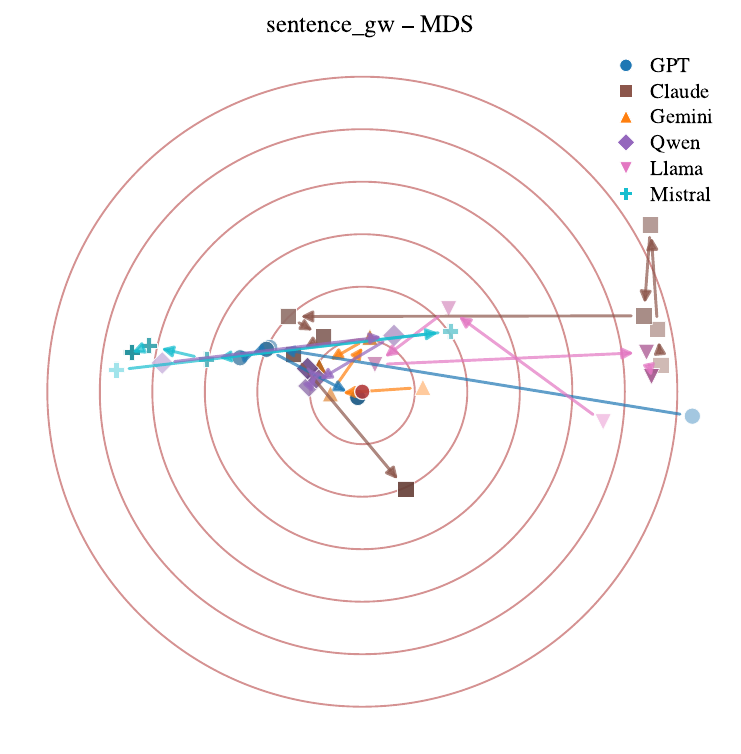}
\caption{Alignment-invariant cartography from Gromov--Wasserstein distance (encoding as
Fig.~\ref{fig:mean}). GW compares each model's \emph{internal} response geometry, so frontier
models pair across families and the family signal weakens. Detail: Fig.~\ref{fig:gw_panels}.}
\label{fig:gw}
\end{figure}

\begin{figure*}[t]
\centering
\begin{minipage}{0.26\textwidth}\centering
\includegraphics[width=\linewidth]{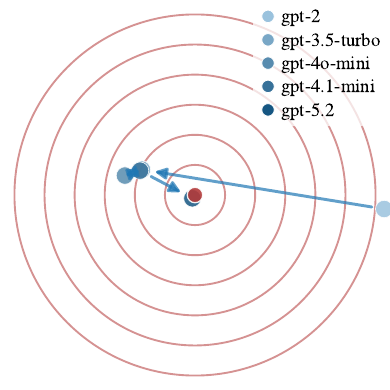}
\end{minipage}\hspace{0.02\textwidth}
\begin{minipage}{0.26\textwidth}\centering
\includegraphics[width=\linewidth]{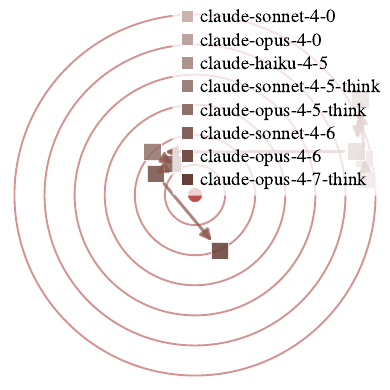}
\end{minipage}\hspace{0.02\textwidth}
\begin{minipage}{0.26\textwidth}\centering
\includegraphics[width=\linewidth]{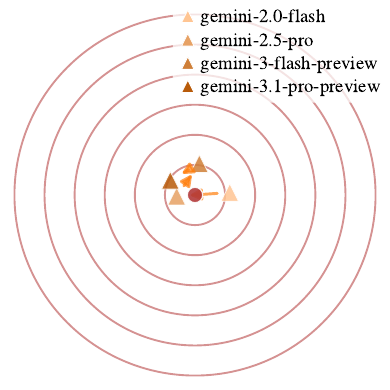}
\end{minipage}\\[3pt]
\begin{minipage}{0.26\textwidth}\centering
\includegraphics[width=\linewidth]{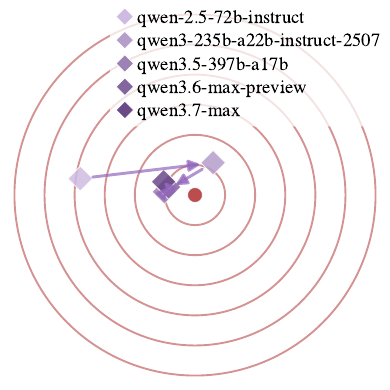}
\end{minipage}\hspace{0.02\textwidth}
\begin{minipage}{0.26\textwidth}\centering
\includegraphics[width=\linewidth]{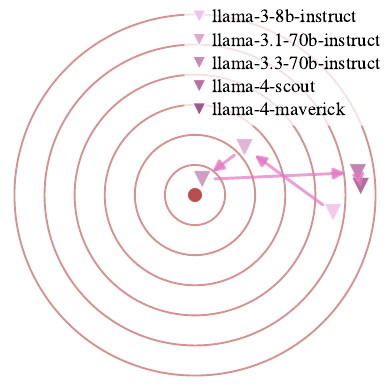}
\end{minipage}\hspace{0.02\textwidth}
\begin{minipage}{0.26\textwidth}\centering
\includegraphics[width=\linewidth]{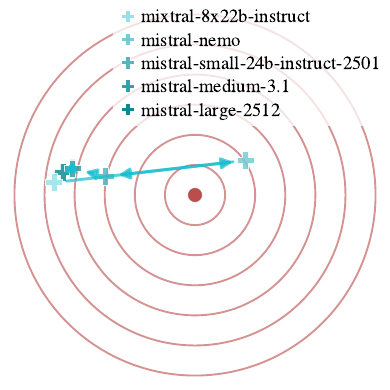}
\end{minipage}
\caption{Per-family detail for Fig.~\ref{fig:gw}: each lineage under Gromov--Wasserstein distance,
drawn in a common frame bounded by the outermost ring (identical axes in every panel). Panels, left
to right and top to bottom: GPT, Claude, Gemini, Qwen, Llama, Mistral.}
\label{fig:gw_panels}
\end{figure*}

\begin{figure*}[t]
\centering
\begin{minipage}{0.26\textwidth}\centering
\includegraphics[width=\linewidth]{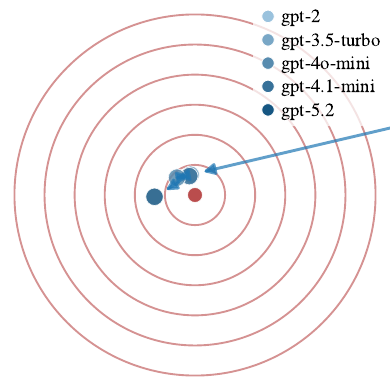}
\end{minipage}\hspace{0.02\textwidth}
\begin{minipage}{0.26\textwidth}\centering
\includegraphics[width=\linewidth]{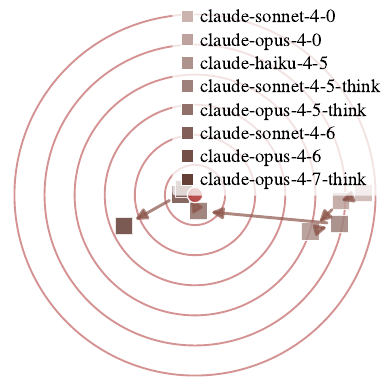}
\end{minipage}\hspace{0.02\textwidth}
\begin{minipage}{0.26\textwidth}\centering
\includegraphics[width=\linewidth]{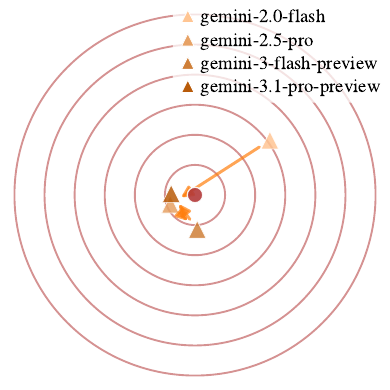}
\end{minipage}\\[3pt]
\begin{minipage}{0.26\textwidth}\centering
\includegraphics[width=\linewidth]{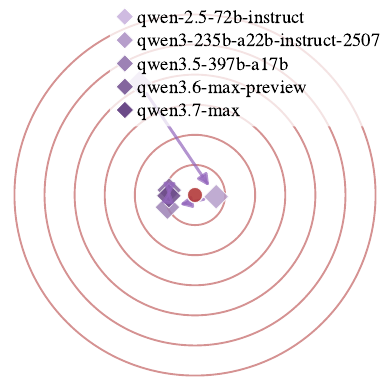}
\end{minipage}\hspace{0.02\textwidth}
\begin{minipage}{0.26\textwidth}\centering
\includegraphics[width=\linewidth]{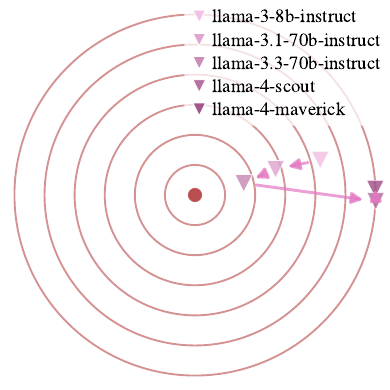}
\end{minipage}\hspace{0.02\textwidth}
\begin{minipage}{0.26\textwidth}\centering
\includegraphics[width=\linewidth]{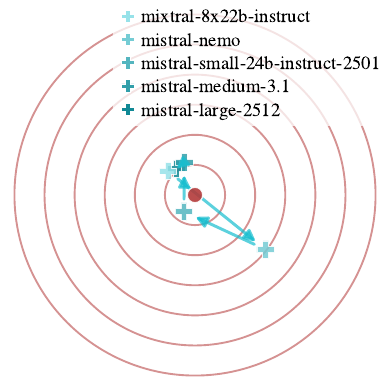}
\end{minipage}
\caption{Per-family detail for Fig.~\ref{fig:tok_map}: each lineage under the token-level per-prompt
MMD distance, drawn in a common frame bounded by the outermost ring (identical axes in every
panel). Panels, left to right and top to bottom: GPT, Claude, Gemini, Qwen, Llama, Mistral.}
\label{fig:tok_panels}
\end{figure*}

\begin{figure*}[t]
\centering
\begin{minipage}{0.49\textwidth}\centering
\includegraphics[width=\textwidth]{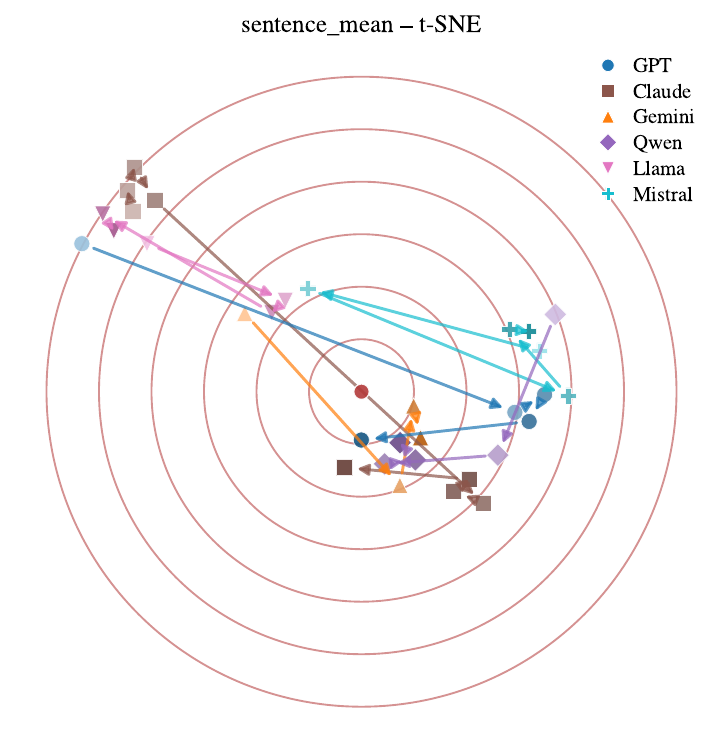}\\[-1pt]
{\small (a) t-SNE}
\end{minipage}\hfill
\begin{minipage}{0.49\textwidth}\centering
\includegraphics[width=\textwidth]{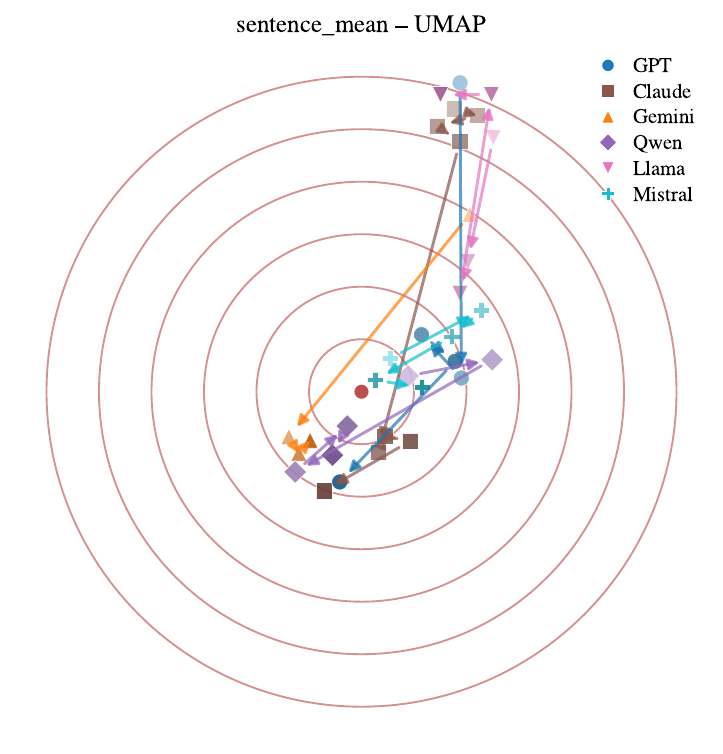}\\[-1pt]
{\small (b) UMAP}
\end{minipage}
\caption{The same mean-metric distances under t-SNE (a) and UMAP (b), with the encoding of
Fig.~\ref{fig:mean}. Family grouping is stable across projections.}
\label{fig:proj}
\end{figure*}

\begin{figure}[t]
\centering
\includegraphics[width=0.55\textwidth]{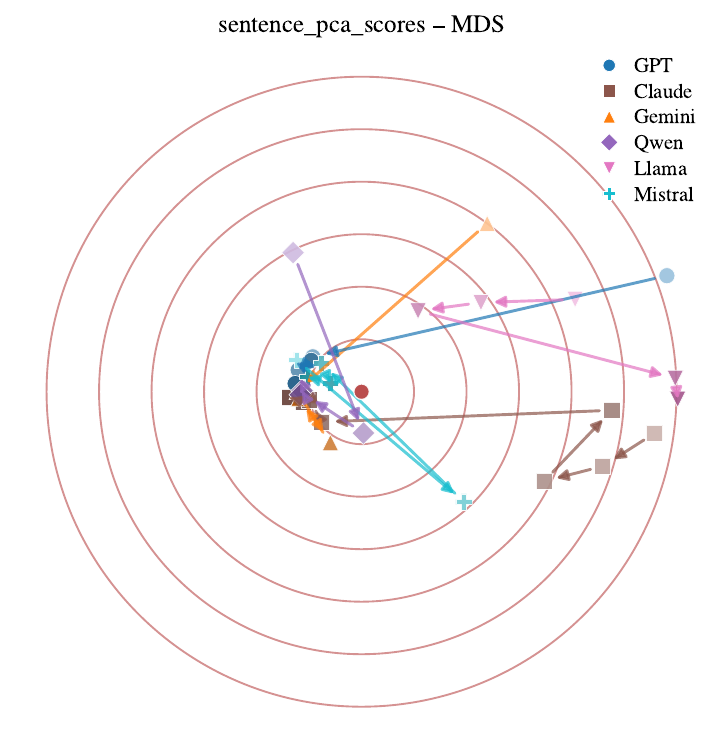}
\caption{The PCA-compressed distance of Eq.~\eqref{eq:dpca} (metric MDS), with the encoding of
Fig.~\ref{fig:mean}. The compression closely matches the mean metric (Spearman $\rho=0.92$).}
\label{fig:pca}
\end{figure}

\begin{figure*}[t]
\centering
\includegraphics[width=0.72\textwidth]{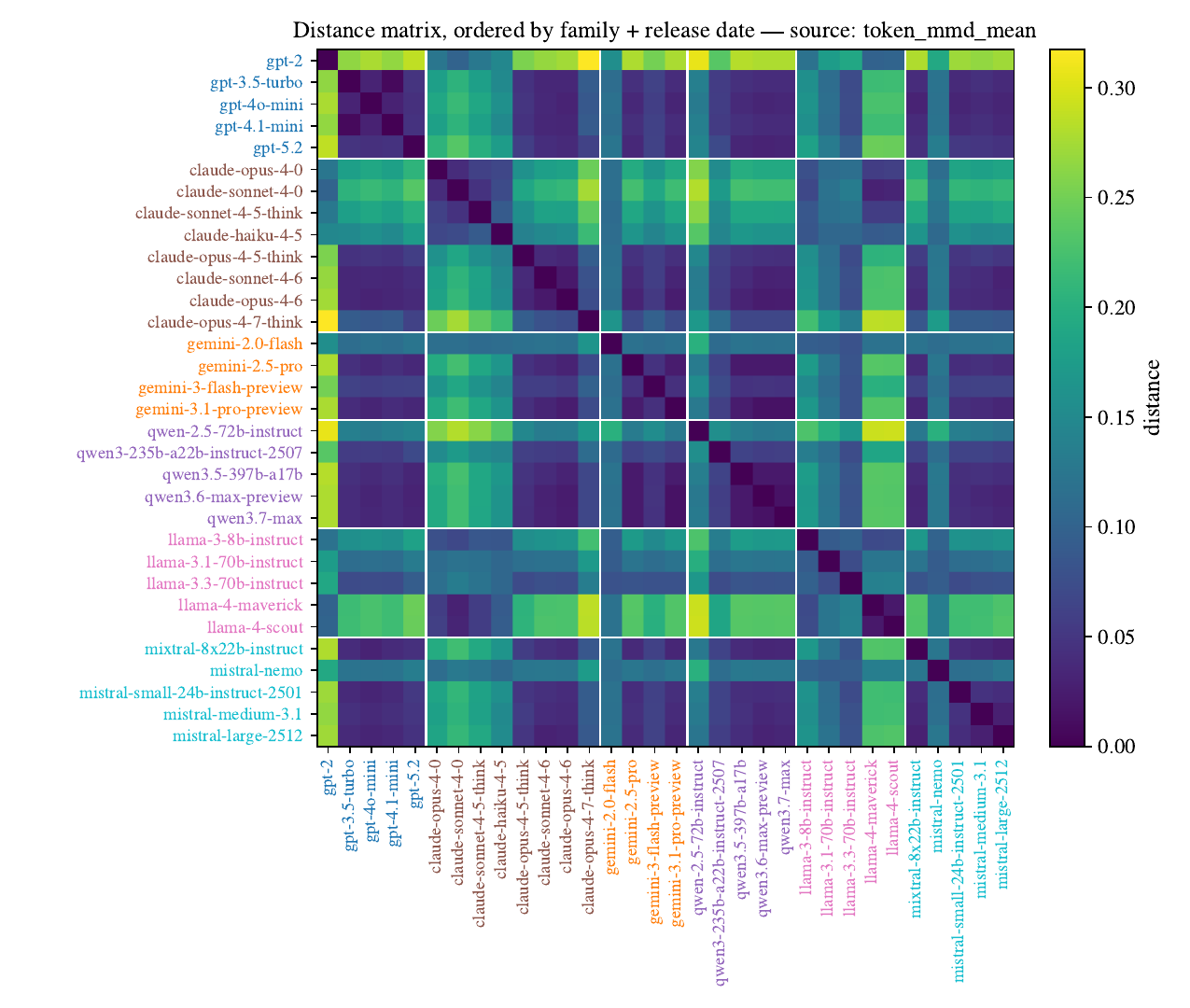}\\[4pt]
\includegraphics[width=0.92\textwidth]{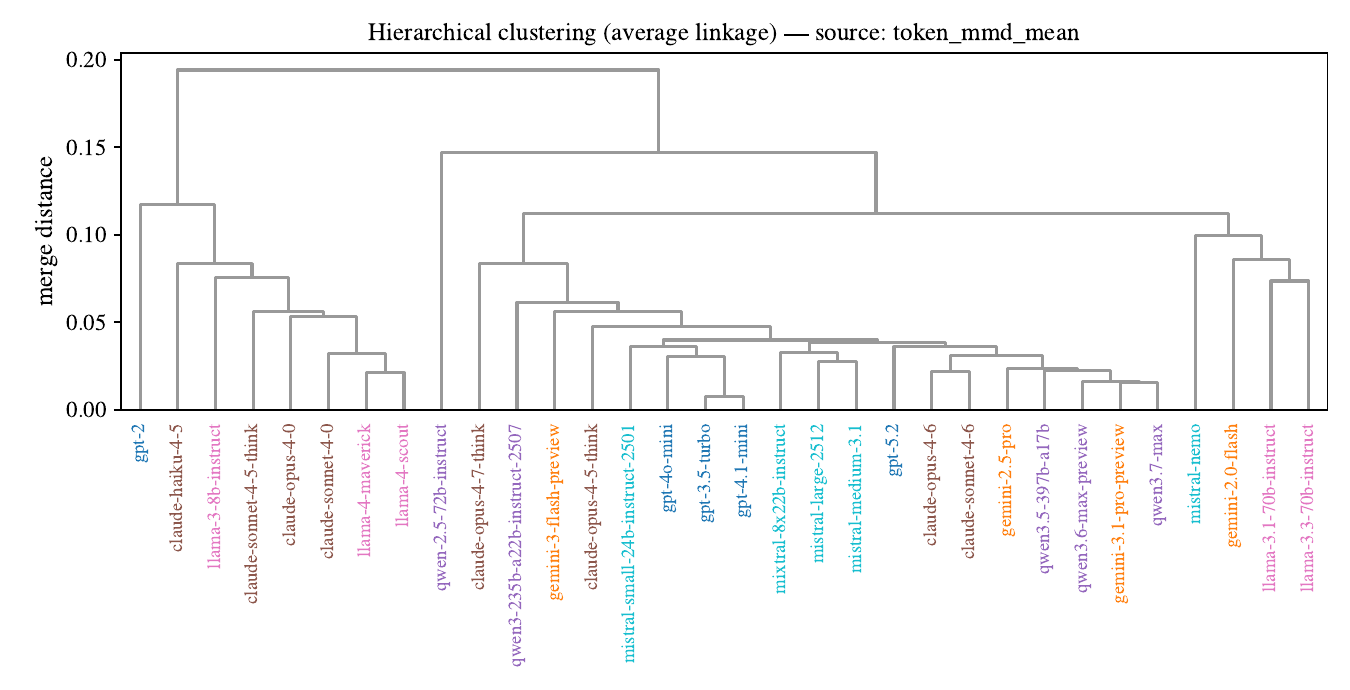}
\caption{Token-level MMD structure: family+date-ordered heatmap (top) and average-linkage
dendrogram (bottom). As at the sentence level, \texttt{gpt-2} and \texttt{qwen-2.5-72b} split off
first and recent models mix across families.}
\label{fig:tok_struct}
\end{figure*}

\begin{figure*}[t]
\centering
\includegraphics[width=\textwidth]{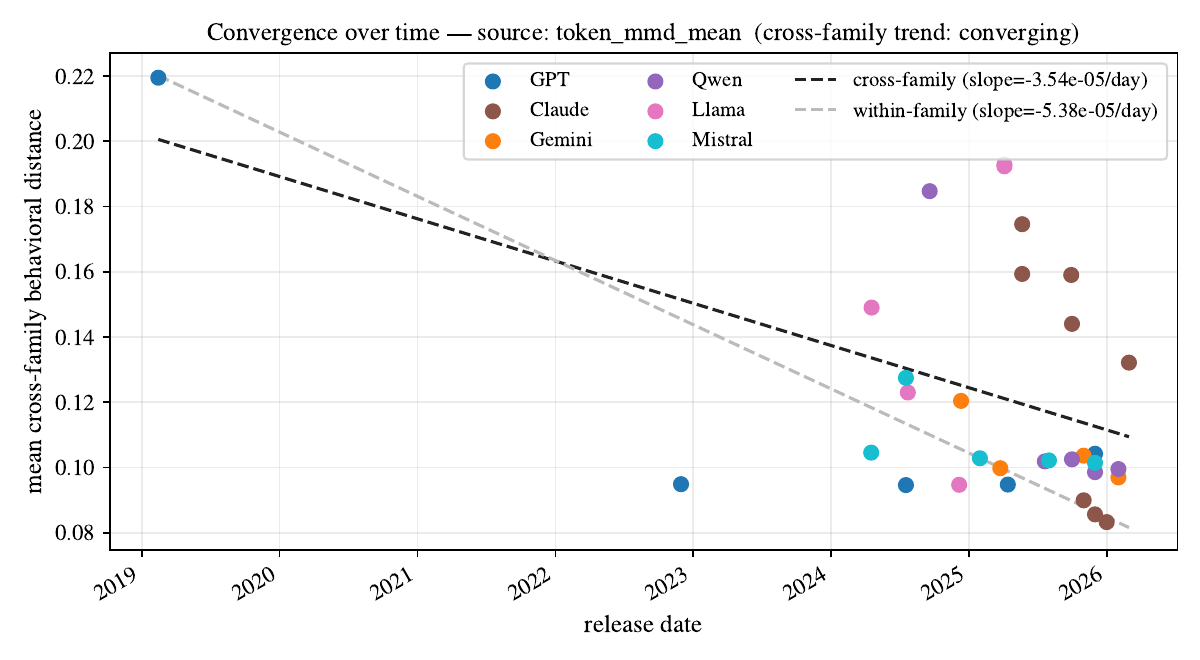}
\caption{Token-level convergence: mean cross-family MMD distance per model versus release date. The
downward trend reproduces the sentence-level frontier homogenization (cf.\ Fig.~\ref{fig:conv}).}
\label{fig:tok_conv}
\end{figure*}

\subsection{Detailed Proofs for the Training-Side Account}
\label{sec:theory_training_behavior}

This section proves Theorem~\ref{thm:behavioral_similarity} and
Corollaries~\ref{cor:semantic_similarity}--\ref{cor:mean_embedding_similarity}.
It also states precisely how the mean-embedding result relates to the
single-response distance used in the experiments.

\paragraph{Setup and conventions.}
Let $\mathcal{R}$ be the countable space of finite token sequences terminated by
an end-of-sequence token, and let $\mathcal{P}_m$ be a joint distribution over
$(X,R)$. We write $\mathcal{P}_{m,X}$ for its prompt marginal and
$p_m(\cdot\mid x)$ for the corresponding conditional response distribution.
The deployed autoregressive model defines
\[
q_m(r\mid x)=\prod_{t=1}^{|r|}q_m(r_t\mid x,r_{<t}),
\]
where the final token in $r$ is the end-of-sequence token. On this countable
space,
$\mathrm{TV}(P,Q)=\sup_A|P(A)-Q(A)|=\tfrac12\sum_r|P(r)-Q(r)|$.
The arguments extend to standard measurable response spaces with regular
conditional distributions. If $q_m(r\mid x)=0$ on a set assigned positive
probability by $p_m(\cdot\mid x)$, the conditional KL divergence is infinite;
the informative case is therefore finite excess risk.

\begin{lemma}[Cross-entropy excess equals conditional KL]
\label{lem:cross-entropy-kl}
For any conditional distribution $q$ with finite risk,
\[
\mathcal{R}_m(q)-\mathcal{R}_m(p_m)
=
\mathbb{E}_{x\sim\mathcal{P}_{m,X}}
\mathrm{KL}\!\left(p_m(\cdot\mid x)\Vert q(\cdot\mid x)\right).
\]
Consequently, $p_m$ minimizes population negative log-likelihood over all
conditional response distributions.
\end{lemma}

\begin{proof}
Conditioning first on $X=x$ and expanding the definition of KL gives
\begin{align*}
&\mathbb{E}_{r\sim p_m(\cdot\mid x)}[-\log q(r\mid x)]
-\mathbb{E}_{r\sim p_m(\cdot\mid x)}[-\log p_m(r\mid x)]\\
&\qquad=
\mathbb{E}_{r\sim p_m(\cdot\mid x)}
\log\frac{p_m(r\mid x)}{q(r\mid x)}
=\mathrm{KL}\!\left(p_m(\cdot\mid x)\Vert q(\cdot\mid x)\right).
\end{align*}
Taking expectation over $x\sim\mathcal{P}_{m,X}$ proves the identity. Since KL
is nonnegative, the minimum is attained at $q(\cdot\mid x)=p_m(\cdot\mid x)$
for $\mathcal{P}_{m,X}$-almost every $x$. For an autoregressive model, the same
identity can equivalently be expanded by the KL chain rule into the sum of
next-token conditional KL terms over all prefixes.
\end{proof}

\begin{lemma}[Transfer from training risk to inference prompts]
\label{lem:risk-tv-transfer}
Suppose Eq.~\eqref{eq:coverage} holds for model $m$, and write
$w_m=d\xi/d\mathcal{P}_{m,X}$. Then
\[
\mathbb{E}_{x\sim\xi}\mathrm{TV}\!\left(
 p_m(\cdot\mid x),q_m(\cdot\mid x)\right)
\leq
\sqrt{\frac{\kappa_m\varepsilon_m}{2}}.
\]
\end{lemma}

\begin{proof}
For each prompt $x$, Pinsker's inequality yields
\[
\mathrm{TV}\!\left(p_m(\cdot\mid x),q_m(\cdot\mid x)\right)
\leq
\sqrt{\frac12
\mathrm{KL}\!\left(p_m(\cdot\mid x)\Vert q_m(\cdot\mid x)\right)}.
\]
Let
$K_m(x)=\mathrm{KL}(p_m(\cdot\mid x)\Vert q_m(\cdot\mid x))$.
Because the square-root function is concave, Jensen's inequality gives
\[
\mathbb{E}_{x\sim\xi}\sqrt{K_m(x)/2}
\leq \sqrt{\mathbb{E}_{x\sim\xi}K_m(x)/2}.
\]
The density-ratio assumption then transfers the KL expectation from the
inference distribution to the training marginal:
\begin{align*}
\mathbb{E}_{x\sim\xi}K_m(x)
&=\int K_m(x)w_m(x)\,d\mathcal{P}_{m,X}(x)\\
&\leq \kappa_m\int K_m(x)\,d\mathcal{P}_{m,X}(x)
=\kappa_m\varepsilon_m,
\end{align*}
where the last equality is Lemma~\ref{lem:cross-entropy-kl}. Combining the
three displays proves the claim.
\end{proof}

\paragraph{Full proof of Theorem~\ref{thm:behavioral_similarity}.}
\begin{proof}
Fix a prompt $x$. Since total variation is a metric on probability measures,
insert the two population conditionals and apply its triangle inequality:
\begin{align*}
\mathrm{TV}\!\left(q_m^x,q_{m'}^x\right)
\leq{}&\mathrm{TV}\!\left(q_m^x,p_m^x\right)
+\mathrm{TV}\!\left(p_m^x,p_{m'}^x\right)\\
&+\mathrm{TV}\!\left(p_{m'}^x,q_{m'}^x\right),
\end{align*}
where $p_j^x=p_j(\cdot\mid x)$ and $q_j^x=q_j(\cdot\mid x)$. Integrating with respect
to $x\sim\xi$, the middle term becomes
$\eta^{\mathrm{txt}}_{m,m'}$. Lemma~\ref{lem:risk-tv-transfer} bounds the first
and third terms by
$\sqrt{\kappa_m\varepsilon_m/2}$ and
$\sqrt{\kappa_{m'}\varepsilon_{m'}/2}$, respectively. Adding the three bounds
gives Eq.~\eqref{eq:tv-behavior-bound}.
\end{proof}

The result has the transparent decomposition
\[
\boxed{
\text{model difference}
\leq
\text{population difference}
+
\text{learning error of }m
+
\text{learning error of }m'.
}
\]
Small excess risks alone are not sufficient: even population-optimal models
can differ when $p_m(\cdot\mid x)$ and $p_{m'}(\cdot\mid x)$ differ on the
evaluation prompts.

\paragraph{Full proof of Corollary~\ref{cor:semantic_similarity}.}
\begin{proof}
For each prompt $x$, the triangle inequality for $W_1$ gives
\[
W_1(\nu_m^x,\nu_{m'}^x)
\leq W_1(\nu_m^x,\pi_m^x)
+W_1(\pi_m^x,\pi_{m'}^x)
+W_1(\pi_{m'}^x,\nu_{m'}^x).
\]
We first bound the two outer terms. For arbitrary response distributions $P,Q$ whose pushforwards are
supported in $\mathcal{S}_\phi$, take a maximal coupling $(U,V)$ satisfying
$\Pr(U\neq V)=\mathrm{TV}(P,Q)$ (which exists on the countable response space).
Because $\operatorname{diam}(\mathcal{S}_\phi)\leq B_\phi$,
\[
\|\phi(U)-\phi(V)\|_2
\leq B_\phi\mathbf{1}\{U\neq V\}.
\]
The pushforward of this coupling is a valid coupling of $\phi_{\#}P$ and
$\phi_{\#}Q$. Therefore, by Eq.~\eqref{eq:w1-definition},
\begin{equation}
W_1(\phi_{\#}P,\phi_{\#}Q)
\leq B_\phi\mathrm{TV}(P,Q).
\label{eq:w1-tv}
\end{equation}
Apply Eq.~\eqref{eq:w1-tv} to the first and third terms, average over
$x\sim\xi$, and invoke Lemma~\ref{lem:risk-tv-transfer} for models $m$ and $m'$.
This yields
\begin{align*}
\mathbb{E}_{x\sim\xi}W_1(\nu_m^x,\nu_{m'}^x)
\leq{}&\eta^{\mathrm{sem}}_{m,m'}
+B_\phi\sqrt{\frac{\kappa_m\varepsilon_m}{2}}\\
&+B_\phi\sqrt{\frac{\kappa_{m'}\varepsilon_{m'}}{2}},
\end{align*}
which is Eq.~\eqref{eq:semantic-behavior-bound}.
\end{proof}

\paragraph{Full proof of Corollary~\ref{cor:mean_embedding_similarity}.}
\begin{proof}
Fix a prompt $x$ and let $\gamma\in\Pi(\nu_m^x,\nu_{m'}^x)$ be any coupling.
Writing $(Z,Z')\sim\gamma$ and using the marginal constraints,
\begin{align*}
\|\mu_m(x)-\mu_{m'}(x)\|_2
&=\|\mathbb{E}_{\gamma}[Z-Z']\|_2\\
&\leq\mathbb{E}_{\gamma}\|Z-Z'\|_2,
\end{align*}
where the inequality follows from Jensen's inequality for the Euclidean norm.
Taking the infimum over all couplings $\gamma$ gives
\begin{equation}
\|\mu_m(x)-\mu_{m'}(x)\|_2
\leq W_1(\nu_m^x,\nu_{m'}^x).
\label{eq:mean-w1}
\end{equation}
Averaging Eq.~\eqref{eq:mean-w1} over $x\sim\xi$ and applying
Corollary~\ref{cor:semantic_similarity} proves
Eq.~\eqref{eq:mean-embedding-bound}.
\end{proof}

\paragraph{Connection to the observed single-response distance.}
Equation~\eqref{eq:dmean} is computed from one generated response per model and
prompt rather than from the conditional means. To state the connection
precisely, define the within-prompt semantic spread
\[
\omega_m=
\mathbb{E}_{x\sim\xi}
\mathbb{E}_{Z\sim\nu_m^x}\|Z-\mu_m(x)\|_2.
\]
Conditional on $x$, let $Z_m\sim\nu_m^x$ and
$Z_{m'}\sim\nu_{m'}^x$ be generated independently. The triangle inequality
through the two conditional means gives
\begin{align*}
\mathbb{E}\|Z_m-Z_{m'}\|_2
\leq{}&\mathbb{E}\|Z_m-\mu_m(x)\|_2
+\|\mu_m(x)-\mu_{m'}(x)\|_2\\
&+\mathbb{E}\|Z_{m'}-\mu_{m'}(x)\|_2.
\end{align*}
After averaging over prompts and using Eq.~\eqref{eq:mean-w1},
\begin{equation}
\begin{aligned}
&\mathbb{E}_{x\sim\xi}
\mathbb{E}\|Z_m-Z_{m'}\|_2\\
&\quad\leq
\eta^{\mathrm{sem}}_{m,m'}
+B_\phi\sqrt{\frac{\kappa_m\varepsilon_m}{2}}
+B_\phi\sqrt{\frac{\kappa_{m'}\varepsilon_{m'}}{2}}
+\omega_m+\omega_{m'}.
\end{aligned}
\label{eq:single-draw-bound}
\end{equation}
Thus, when responses are sampled from the model distributions, the empirical
mean per-prompt distance estimates a quantity controlled by population
similarity, learning errors, and within-prompt generation variability. The
$\omega$ terms are essential in general: an independent-sample transport cost
need not equal the optimal-coupling cost $W_1$.

\paragraph{Bounded downstream scores.}
For any measurable score $s(x,r)\in[0,1]$, the variational characterization of
total variation gives, for each $x$,
\[
\left|\mathbb{E}_{q_m(\cdot\mid x)}s(x,R)
-\mathbb{E}_{q_{m'}(\cdot\mid x)}s(x,R)\right|
\leq\mathrm{TV}\!\left(q_m(\cdot\mid x),q_{m'}(\cdot\mid x)\right).
\]
Averaging over $x\sim\xi$ and applying
Theorem~\ref{thm:behavioral_similarity} shows that behaviorally close models
have close expected values for every bounded evaluation score. This does not
imply high capability: the two models may be similarly incorrect.

\paragraph{Effect of training-set size.}
The theorem is stated in terms of population excess risk. A finite-sample
learning analysis may, under an appropriate oracle inequality for the chosen
model class and algorithm, yield a bound of the schematic form
\[
\varepsilon_m
\leq a_m+\tau_m+g_m(n_m,\delta),
\]
where $a_m$ is approximation error, $\tau_m$ is optimization error, and
$g_m(n_m,\delta)$ is a generalization term that decreases with sample size
$n_m$. Substituting such a bound into Theorem~\ref{thm:behavioral_similarity}
shows how more data can reduce the finite-sample contribution. This is not a
universal decomposition without additional capacity and algorithmic
assumptions, and large datasets alone do not eliminate approximation,
optimization, population-discrepancy, or coverage errors.

\paragraph{Scope and limitations of the sufficient condition.}
The theorem directly models log-loss training. Pretraining data alone may not
capture instruction tuning, preference optimization, safety training, or
system-level policies; applying the theorem to a post-trained system therefore
requires a separate argument that its deployed distribution has small excess
log-loss relative to an appropriate effective target population. The coverage assumption is also substantive: if the evaluation prompts
place mass outside the support of the effective training marginal, no finite
$\kappa_m$ exists and the theorem gives no guarantee.

Finally, the results concern probabilistic response distributions. They apply
directly when evaluation responses are sampled from $q_m$ (and remain valid
after a common response post-processing Markov kernel, since total variation
contracts under such a kernel). Greedy decoding is a deterministic functional of the probability vector rather
than a sample from it and can switch outputs discontinuously when the leading probabilities are nearly tied. A direct guarantee for greedy
generation therefore needs an additional margin or decoder-stability
assumption. Under these qualifications, decreasing excess risks and decreasing
$\eta^{\mathrm{sem}}_{m,m'}$ imply semantic convergence; a positive limiting
population discrepancy or irreducible learning error instead yields relative
homogenization toward a positive floor.
\end{document}